\documentclass{article}

\PassOptionsToPackage{numbers, compress}{natbib}
\input{commands.sty}

   \usepackage[preprint]{neurips_2025}

\usepackage[utf8]{inputenc} 
\usepackage[T1]{fontenc}    
\usepackage{hyperref}       
\usepackage{url}            
\usepackage{booktabs}       
\usepackage{amsfonts}       
\usepackage{nicefrac}       
\usepackage{microtype}      
\usepackage{xcolor}         

\title{The Cost of Compression: Tight Quadratic Black-Box Attacks on Sketches for $\ell_2$ Norm Estimation}

\author{
  Sara Ahmadian \\
  Google Research \\
  \texttt{sahmadian@google.com} \\
  \And
  Edith Cohen \\
  Google Research and Tel Aviv University \\
  \texttt{edith@cohenwang.com} \\
  \And
  Uri Stemmer \\
  Tel Aviv University and Google Research\\
  \texttt{u@uri.co.il} \\
}

\begin{document}
\maketitle

\begin{abstract}

Dimensionality reduction via linear sketching is a powerful and widely used technique, but it is known to be vulnerable to adversarial inputs. We study the \emph{black-box adversarial setting}, where a fixed, hidden sketching matrix $A \in \mathbb{R}^{k \times n}$ maps high-dimensional vectors $\boldsymbol{v} \in \mathbb{R}^n$ to lower-dimensional sketches $A\boldsymbol{v} \in \mathbb{R}^k$, and an adversary can query the system to obtain approximate $\ell_2$-norm estimates that are computed from the sketch.

We present a \emph{universal, nonadaptive attack} that, using $\tilde{O}(k^2)$ queries, either causes a failure in norm estimation or constructs an adversarial input on which the optimal estimator for the query distribution (used by the attack) fails. The attack is completely agnostic to the sketching matrix and to the estimator—it applies to \emph{any} linear sketch and \emph{any} query responder, including those that are randomized, adaptive, or tailored to the query distribution.

Our lower bound construction tightly matches the known upper bounds of $\tilde{\Omega}(k^2)$, achieved by specialized estimators for Johnson–Lindenstrauss transforms and AMS sketches. Beyond sketching, our results uncover structural parallels to adversarial attacks in image classification, highlighting fundamental vulnerabilities of compressed representations.

\end{abstract}

\section{Introduction}

Dimensionality reduction is a fundamental technique in data analysis, algorithm design, and machine learning. A common paradigm is to apply a \emph{sketching map}, a compressive transformation $C: \mathbb{R}^n \to \mathbb{R}^k$, which maps a high-dimensional input vector to a lower-dimensional representation. The map is typically sampled from a known distribution (e.g., The Johnson–Lindenstrauss (JL) transform~\citep{johnson1984extensions} ) or learned during training, and is designed to support estimation of specific properties of the input $\boldsymbol{v}$, such as norms, inner products, or distances, using only the sketch $C(\boldsymbol{v})$.

Once constructed, the sketching map $C$ typically remains \emph{fixed} across all inputs. This is true when compression layers are part of a trained model, and it is necessary in algorithmic contexts that require \emph{composability}—the ability to sketch distributed datasets independently and combine the results—without revisiting the original data. A fixed sketching map also supports downstream tasks that operate directly in sketch space.

However, this compression introduces an intrinsic vulnerability: small input perturbations can produce large changes in the sketch, even when the true property (such as the norm) is nearly unchanged. For linear sketching maps $A \in \mathbb{R}^{k \times n}$, this vulnerability arises from structural facts such as the existence of nontrivial null-space vectors and directions along which small-norm perturbations yield large distortions.

Such \emph{adversarial inputs} can be constructed easily in the white-box setting, where the sketching map $C$ is known. In this work, we focus on the  \emph{black-box attack model}, which captures realistic situations where the adversary does not see the sketching map directly but can access it interactively via a \emph{responder} using queries of the following form:

\begin{enumerate}[label=(\roman*)]
    \item The adversary chooses a query vector $\boldsymbol{v} \in \mathbb{R}^n$.
    \item The system receives $\boldsymbol{v}$ and computes a sketch $C(\boldsymbol{v})$.
    \item The responder selects a map $\psi$ of sketches to distributions, receives the sketch $C(\boldsymbol{v})$ from system and returns a (possibly randomized) value $s \sim \psi(C(\boldsymbol{v}))$ to the adversary.
    \end{enumerate}
The goal of the adversary, in general terms, is to compromise the sketching map by causing a failure: If responses are correct, construct an adversarial input.
We measure the efficiency of such \emph{attacks} by the \emph{number of queries} they require, as a function of the sketch size $k$.

We distinguish between two settings. In the \emph{nonadaptive} case, all queries (including the final candidate) are chosen without regard to responses from earlier queries. The adversary performs a blind search, and success depends on the density of adversarial directions in the input space. In contrast, in the \emph{adaptive} setting, each query can depend on prior responses, allowing the adversary to extract information about the sketching map and potentially converge to an adversarial input more efficiently.

The black-box model in the adaptive setting is well studied across multiple areas, including statistical queries~\citep{Freedman:1983,Ioannidis:2005,FreedmanParadox:2009,HardtUllman:FOCS2014,DworkFHPRR15,BassilyNSSSU:sicomp2021}, sketching and streaming algorithms~\citep{MironovNS:STOC2008,HardtW:STOC2013,BenEliezerJWY21,HassidimKMMS20,WoodruffZ21,AttiasCSS21,BEO21,DBLP:conf/icml/CohenLNSSS22,TrickingHashingTrick:AAAI2023,AhmadianCohen:ICML2024,GribelyukLWYZ:FOCS2024,CNSSS:ArXiv2024}, dynamic graph algorithms~\citep{ShiloachEven:JACM1981,AhnGM:SODA2012,gawrychowskiMW:ICALP2020,GutenbergPW:SODA2020,Wajc:STOC2020, BKMNSS22}, and machine learning~\citep{szegedy2013intriguing,goodfellow2014explaining,athalye2018synthesizing,papernot2017practical,MadryRobustness:ICLR2018,Salman:NEURIPS2019}.

In this study, we aim to better understand this vulnerability for the task of $\ell_2$-norm estimation and to gain more general insights through this lens. Specifically, we consider linear transformations specified by a \emph{sketching matrix} $A \in \mathbb{R}^{k\times n}$ which maps $\boldsymbol{v}\in \mathbb{R}^{n}$ to a \emph{sketch} $A \boldsymbol{v} \in \mathbb{R}^{k}$. 

Two classic methods for this task are 
the Johnson–Lindenstrauss (JL) transform~\citep{johnson1984extensions} and the AMS sketch~\citep{ams99}, They define \emph{distributions} over sketching matrices that approximately preserve Euclidean norms (and therefore support sketch-based approximation of distances).
The provided guarantees are probabilistic: for any input vector, with high probability over the random choice of the sketching matrix $A \in \mathbb{R}^{k\times n}$, the scaled norm of the sketch $\|A \boldsymbol{v}\|_2$ closely approximates $\|\boldsymbol{v}\|_2$ (within relative error $\epsilon$ with probability $1-\delta$ when $k=O(\epsilon^{-2}\log(1/\delta)$).

However, as said, in practice, the sketching matrix is typically \emph{fixed} over all input vectors, and importantly, 
no fixed matrix can preserve approximate norms for \emph{all} inputs: every sketching matrix inevitably has inputs on which it fails.
\ignore{
When the sketching matrix is known, these ``bad'' inputs can often be found easily. There had been significant interest in black-box attacks where the query responder gets the sketch $A\boldsymbol{v}$ (but not $\boldsymbol{v}$) and based on it returns an approximation of the norm $\|\boldsymbol{v}\|_2$ to the adversary. 

The question of robustness is how hard it is to find such bad inputs, when there is only a \emph{black-box} query access to the matrix. Specifically, an \emph{adversary} does not see the matrix directly but may query it through the following interactive steps:
\begin{enumerate}[label=(\roman*)]
    \item The adversary constructs a query vector $\boldsymbol{v}$.
    \item A \emph{query responder} gets the sketch $A\boldsymbol{v}$ (but not $\boldsymbol{v}$) and based on it returns an approximation of the norm $\|\boldsymbol{v}\|_2$ to the adversary.
    \end{enumerate}
    The goal of the adversary is to cause estimation to fail, and 
the efficiency of the attack is measured by the number of queries they require as a function of the sketch size $k$.

}
In the nonadaptive setting, 
the guarantees of JL and AMS apply via a union bound: with high probability (at least $1 - \delta$) over the sampling of matrices, the fixed sampled sketching matrix supports up to $\delta e^{O(k\epsilon^2)}$ approximate norm queries, in the sense that with probability $1 - \delta$, they are all accurate to within relative error $\epsilon$. Therefore, finding a bad input requires a number of queries that is exponential in $k$.

The adaptive setting was studied in multiple works.
In terms of positive results, it is known that JL and AMS sketches can 
effectively trade off baseline (non-adaptive) accuracy and robustness within a fixed sketch size budget of $k$. By using 
carefully designed ``robust'' estimators~\citep{HassidimKMMS20,Blanc:STOC2023}, the sketch can  support for a fixed $\epsilon$, a number of adaptive queries that is \emph{quadratic} in $k$. The idea, based on \citep{DworkFHPRR:STOC2015,BassilyNSSSU:sicomp2021}, is to protect information on the sketching matrix by adding noise to the returned ``best possible'' estimate (or by subsampling a part of the sketch for each response). 

As for negative results,
\citet{HardtW:STOC2013} constructed an attack of size polynomial in $k$ that for any sketching matrix constructs a distribution over inputs under which any estimator would fail. The idea in the attack is to identify vectors that lie close to the null space of $A$ and hence are transparent to the sketch.
\citet{DBLP:conf/nips/CherapanamjeriN20,BenEliezerJWY21} presented an attack of size linear in $k$ on the JL and AMS sketches with the standard estimator (which returns a scaled norm of the sketch). The product of these attack is a vector on which the standard estimator fails.
\citet{TrickingHashingTrick:AAAI2023} constructed an attack of quadratic size on the AMS sketch that is \emph{universal} (applies against any query responder), that constructs a vector on which the standard estimator fails. 
These negative results vary by (i) the scope of the sketching matrices compromised (general or JL/AMS), (ii) the power of the query responder (strategic and adaptive or the standard estimator), and (iii)~the product of the attack (a distribution that fails any query responder or a single vector with an out-of-distribution sketch that fails the ``optimal'' estimator). 

An intriguing gap remains, however, between the established quadratic guarantee for the number of adaptive queries with correct responses and the super-quadratic sizes of the known attacks for general sketching matrices. 

\subsection{Overview of contributions}

Our primary contribution in this work, is a construction of an attack of quadratic size in $k$ that applies against \emph{any} sketching matrix $A \in \mathbb{R}^{k \times n}$ and with any query responder for  $\ell_2$ norm estimation. The attack produces a vector with an out-of-distribution sketch:

\begin{theorem}[Attack Properties]
There exist a universal constant $C > 0$, and families of distributions $\mathcal{F}_n$ over $\mathbb{R}^n$, such that the following holds.

For every sketching matrix $A \in \mathbb{R}^{k \times n}$ with $n = \Omega(k^2)$, any query responder, and deviation $\gamma \geq 1$ and accuracy $\alpha\in (0,1)$ parameters, 
with probability at least $0.9$ over the choice of a distribution $\mathcal{D} \sim \mathcal{F}$,
 after $r = C \gamma^2 \alpha^{-2} k^2 \log^2 k$ i.i.d. queries $\boldsymbol{v} \sim \mathcal{D}$, one of the following outcomes occurs:
\begin{enumerate}[label=(\roman*)]
    \item At least $\delta(\alpha)>0$\footnote{A constant that depends on $\alpha$} fraction of the responses have relative error greater than $\alpha$, or
    \item We construct a query vector on which the optimal estimator (with respect to $A$ and $\mathcal{D}$) returns a value that deviates from the true norm by at least a multiplicative factor of $\gamma$.
\end{enumerate}
\end{theorem}

Our attack deploys a simple and natural query distribution that combines a weighted sparse signal with additive dense noise:
\begin{equation*}
    \boldsymbol{v} = w \boldsymbol{e}_h + \boldsymbol{u},
\end{equation*}
where $\boldsymbol{e}_h$ is the standard basis vector corresponding to index $h \in [n]$, and $\boldsymbol{u}$ is Gaussian noise supported on a set $M \subset [n] \setminus \{h\}$ of size $m > k^2$.

Importantly, sparse queries alone are insufficient for attack: under the mild assumption that the sketching matrix $A$ has its columns in general position, the $\ell_2$ norm of any $k$-sparse vector can be exactly recovered from its sketch (e.g., \citep{CandesTao2005decoding}). This implies that if the attack is restricted to sparse inputs, the response leaks no information about the sketching matrix and full robustness is preserved.

To sample from our query distribution, we first choose the \emph{signal coordinate} $h \sim [n]$ uniformly at random, and then independently sample a \emph{noise support} $M \subset [n] \setminus \{h\}$ of size $m = \Omega(k^2)$.


The goal of the query responder is to estimate the $\ell_2$ norm of $\boldsymbol{v}$ from its sketch $A\boldsymbol{v}$, with relative error at most $\alpha$. Our attack remains effective even against the simpler \emph{norm gap} task: return $-1$ if $\| \boldsymbol{v} \|_2 \leq 1$ and return $1$ if $\| \boldsymbol{v} \|_2 \geq 1+\alpha$, with either output permitted when $\| \boldsymbol{v} \|_2 \in (1, 1+\alpha)$. 

Note that the norm gap task reveals strictly less information than norm estimation: a norm estimate with relative error $\Theta(\alpha)$ trivially yields a correct solution to the norm gap task with parameter $\Theta(\alpha)$, but not vice versa.


Our attack is described in \cref{algo:attack}.
Query vectors $(\boldsymbol{v}^{(t)})_{t\in[r]}$ are constructed by sampling a signal value $w^{(t)}\sim W$, where $W$ is a probability distribution over $\mathbb{R}$, and Gaussian noise $\boldsymbol{u}^{(t)}$ to form
$\boldsymbol{v}^{(t)} = w^{(t)}\boldsymbol{e}_h+ \boldsymbol{u}^{(t)}$.
The adversary collects the responses 
$s^{(t)}$ for the sketch $A\boldsymbol{v}^{(t)}$.
We establish that if the responses are correct for the norm gap (except for a small fraction of queries) then the normalized signed sum of the noise vectors $\sum_t s^{(t)}\boldsymbol{u}^{(t)}$ is adversarial.

\ignore{
*******

With our construction, we use noise vectors with expected norm $1$.
The squared norm of $\boldsymbol{v}^{(t)}$ is  very concentrated around $w^2+1$ and our norm gap problem is essentially a signal estimation problem. 

For signal value $w$, the sketch distribution is a $k$-dimensional Gaussian with means $w A_{bullet h}$. 
For our query distribution, there is an unbiased and complete sufficient statistic $T(A\boldsymbol{v})$ for the signal $w$ from the sketch which also minimizes the squared error. 
The constructed adversarial vector is a vector in the noise sphere for which this statistic returns a value that is larger than $\gamma(1+w^2)$.


Sampled JL and AMS matrices support responders that are accurate with probability $1-\delta$ on $r$ queries for all features $h\in [n]$ and sampled noise supports $M$, 
with $\alpha = \Omega(\sqrt{k\log(r/\delta)})$. 

We show that this guarantee turn out to be tight in that for \emph{any} fixed $A\in \mathbb{R}^{k\times n}$ such that $n = \Omega(k\log^2  k)$, for most indices $h$ and noise supports with $|M|=\Omega(k\log^2 k)$, any estimator (even ones tailored for $h$ and $M$), the maximum error is $\Omega(\sqrt{k\log(r/\delta)})$.

This property is what facilitates our attack on general matrices. If we randomly sample the signal location we are likely to hit one with high $\sigma_T$.

Analysis idea: Deviation accumulates faster than the norm. Property that was used in other attacks in more limited settings.

******

}

\paragraph{Universality and Limitations}
Our attack is \emph{universal} in that it applies with \emph{any} query responder. The analysis allows the responder to be strategic and adaptive, with full knowledge of the query distribution $\mathcal{D}$ and the internal state of the attacker.
Notably, the attack is \emph{single batch} -- it uses adaptivity in the minimal possible way: all queries are generated independently and only the final adversarial vector is constructed adaptively from the responses.\footnote{Single batch attacks were constructed in prior works for JL, Count-Sketch, AMS, and Cardinality sketches~\citep{DBLP:conf/nips/CherapanamjeriN20,DBLP:conf/icml/CohenLNSSS22,CNSS:AAAI2023,AhmadianCohen:ICML2024}.}

A limitation of our result is that the attack guarantees failure only for the \emph{optimal estimator} tailored to $\mathcal{D}$ and $A$, rather than for \emph{every} possible query responder. The stronger goal—constructing a distribution over inputs that defeats all responders with high probability—remains open.

Despite this, we believe our result is both theoretically meaningful and practically relevant. Theoretically, we obtain a \emph{tight} quadratic bound in the batch-query model, whereas known attacks with the stronger guarantee require significantly more queries (a higher degree polynomial in $k$). Moreover, any attack achieving the stronger guarantee would require at least $\tilde{\Omega}(k)$ \emph{adaptive batches}, and thus cannot be realized within our single-batch setting. In this sense, the adversarial vector we construct is the strongest outcome achievable in our setting.

Our attack product is practically relevant because the optimization process in a model training tends to converge to an (at least locally) optimal estimator for the training distribution. Therefore, compromising the model means compromising this specific implemented estimator, coded in the model parameter values, rather than any possible query responder.

\ignore{
=====
Our attack is \emph{universal} in that it applies against \emph{any} query responder. Our analysis allows the responder to be strategic, adaptive, and be tailored to the query distribution $\mathcal{D}$ and the internal state of the attacker.
Additionally, our attack does not use full fledged adaptivity: all queries are generated independently in a single batch and the only adaptively generated query is the final adversarial vector.

One limitation is our condition on the product of the attack, of returning an adversarial vector for an optimal estimator. This is modest compared with the ultimate goal, of constructing a distribution on inputs for which \emph{any} responder would have high failure probability. 

We argue that it is nonetheless both interesting and relevant. It is theoretically interesting because for this condition we do obtain a \emph{tight} quadratic bound whereas we are only aware of a high degree polynomial with the stronger condition. We also know that the ultimate goal would require at least $\tilde{\Omega}(k)$ batches, and in particular, is not possible in a single batch attack like ours. So in a sense this product is the best we can hope for with such an attack. Additionally, we believe the quadratic bound is the right answer and our techniques may be useful towards that.  Our attack product is relevant because the optimization process in a model training process tends to converge to an (at least locally) optimal estimator for the training distribution. Therefore, compromising this learned estimator --  rather than any possible query responder -- is the goal of adversarial attacks. 
}

\paragraph{Roadmap}
Our attack algorithm is described in \cref{attack:sec} with the analysis presented in \cref{analysis:sec}. An empirical study of our attack on JL and AMS sketching matrices is included in \cref{empirical:sec}. We conclude in \cref{conclusion:sec} with a discussion of open directions and implications to image classifiers.



\ignore{
\eccomment{perhaps embed the following in main text. Since essentially our empirical section uses such a robustified estimator}
A generic way to add robustness is via \emph{wrapper} methods that use in a black box manner multiple independent copies of a randomized data structure that only provides statistical guarantees for non-adaptive queries. A basic wrapper supports $r$ adaptive queries by maintaining $r$ copies and uses a fresh copy for each query. Advanced approaches support a \emph{quadratic} in the number of copies $r$. This was first established for adaptive statistical queries~\cite{Kearns1998} over a sample of size $r$ by
\citep{DworkFHPRR:STOC2015,BassilyNSSSU:sicomp2021}, who use differential privacy to protect the identity of the sampled keys. This was lifted by~\citet{HassidimKMMS20} to general randomized structures including sketches. An alternative method that does not rely on differential privacy analysis was proposed by~\cite{Blanc:STOC2023}. A challenge for sketching problems is to understand when is this quadratic bound is tight. Specifically for the JL transform, 
it consists of $k =O(\epsilon^{-2}\log(1/\delta))$ i.i.d.\ linear measurements, sampled from a symmetric (to permutation of the coordinates) distribution. Maintaining $r$ independent copies is the same as a single larger sketching matrix $A\in \mathbb{R}^{kr}$. 
}

\section{Preliminaries} \label{prelims:sec}

We denote vectors in boldface $\boldsymbol{v}$,  scalars $v$ in non boldface, and inner product of two vectors by  $\langle \boldsymbol{v},\boldsymbol{u} \rangle =\sum_i v_i u_i$. For a vector $\boldsymbol{v}\in \mathbb{R}^n$, we refer to $i\in[n]$ as a {\em key} and to $v_i$ as the value of the $i$th key. 
For $M\subset [n]$ let $v_M$ be the projection of $\boldsymbol{v}$ on entries $(v_i)_{i\in M}$.
To streamline the presentation, we interchangeably use the same notation to refer to both a random variable and a distribution. 

$\mathcal{N}(\mu,\sigma^2)$ is the normal distribution with mean $0$ and variance $\sigma^2$ with density function 
\begin{equation} \label{pdfnormal:eq}
\varphi_{\mu,\sigma^2}(x) := \frac{1}{\sigma\sqrt{2\pi}} e^{-\frac{1}{2} \left(\frac{x-\mu}{\sigma}\right)^2}\ .
\end{equation}
$\mathcal{N}_\ell(0,\sigma^2)$ is the $\ell$-dimensional Gaussian distribution with covariance matrix $I_\ell$ ($\ell$ i.i.d.\ $\mathcal{N}(0,\sigma^2)$). Its probability density function is
\begin{equation} 
f_{\sigma}(\boldsymbol{u}) = 
\frac{1}{(\sigma\sqrt{2\pi})^\ell} e^{-\frac{\|\boldsymbol{u}\|^2_2}{2\sigma^2} }\ .
\end{equation}

A \emph{linear sketching map} is 
defined by a \emph{sketching matrix} $A \in \mathbbm{R}^{k\times n}$ where $k\ll n$.
The input is represented as a vector $\boldsymbol{v}\in \mathbbm{R}^n$ and the \emph{sketch} of $\boldsymbol{v}$ is the product $A \boldsymbol{v} \in \mathbbm{R}^k$.

\section{Attack description} \label{attack:sec}

\begin{definition}[$(y,\alpha)$-Gap Problem] \label{gap:def}
Given width parameter $\alpha>0$ and $y\in\mathbb{R}$, a \emph{gap problem} is, for input $x\in\mathbb{R}$ to return $-1$ when $x \leq y$ and to return $1$ when $x \geq y+\alpha$. The output may be arbitrary  in $\{-1,1\}$ if $\|\boldsymbol{v}\|_2 \in (y,y+\alpha)$.
\end{definition}
Observe that an additive approximation of $\alpha/2$ or a multiplicative approximation of $\alpha/(2(y+\alpha)$ for $x$ yields a correct solution to the $(\ell,\alpha)$-gap problem for $x$. Hence an attack that is effective with the weaker norm gap responses is more powerful.

\begin{algorithm2e}
{\small
    \caption{\texttt{Universal Attack on Sketching Matrix $A$; $\ell_2$ norm gap responder}}
    \label{algo:attack}
    \DontPrintSemicolon
    \KwIn{$A\in \mathbb{R}^{k\times n}$, accuracy parameter $\alpha$, number of queries $r$, signal index $h\in [n]$, support $M\subset [n]$ of size $m=|M|$, noise scale factor $c$}
 \For(\tcp*[h]{Main loop}){$t\in [r]$}{
 $z^{(t)}_i \sim 
\begin{cases}
0 & \text{if } i \notin M \\
 \mathcal{N}(0, 1/m) & \text{if } i \in M
\end{cases}$ \tcp*{sample noise vector}
 $w^{(t)}\sim W$ \tcp*{Sample signal weight from $W$ (\cref{signaldensity:def})}
 $\boldsymbol{v}^{(t)} = w^{(t)} \cdot \boldsymbol{e}_h +  c\boldsymbol{z}^{(t)}$\tcp*{Query vector}
 Responder chooses an estimator $\psi^{(t)}:\mathbb{R}^k\to \mathcal{P}\{-1,1\}$ 
 \tcp*{Map sketches to responses}
 $s^{(t)} \sim \psi^{(t)}(A \boldsymbol{v}^{(t)})$\tcp*{Responder receives sketch, returns $s^{(t)}\in\{-1,1\}$}
} 
     \Return{$\boldsymbol{z}^{(\mathrm{adv})} \gets \frac{\sum_{t\in[r]} s^{(t)}\boldsymbol{z}^{(t)}}{ \|\sum_{t\in[r]} s^{(t)}\boldsymbol{z}^{(t)}\|_2}$}\tcp*{Adversarial noise}
     }
  \end{algorithm2e}


Our attack is described in \cref{algo:attack}. 
The signal density and parameter settings are described in \cref{signaldensity:def}.

In each of $r$ attack steps,
\begin{enumerate}
    \item The 
adversary samples an independent query vector
$\boldsymbol{v}:= w \boldsymbol{e}_h + c\boldsymbol{z}$ by sampling Gaussian noise with support $M$
$\boldsymbol{z}\sim \mathcal{N}_{n,M}(0,\sigma^2=1/m)$, sampling
signal weight $w\sim W$.
    \item The responder chooses an estimator map $\psi:\mathbb{R}^k \to \mathcal{P}(\{-1,1\})$, obtains the sketch $A\boldsymbol{v}$ from the system, and returns $s\sim \psi(A\boldsymbol{v})$ to the adversary.
    \item The adversary then adds $s \boldsymbol{z}$ to the accumulated output.
\end{enumerate}
The product of the attack $\boldsymbol{z}^{\mathrm{(adv)}}$ is the normalized accumulated output.

\begin{restatable}[Attack efficacy]{theorem}{attackefficacy}
\label{attackefficacy:thm}
    Let $A\in\mathbb{R}^{k\times n}$ be a sketching matrix. Consider applying \cref{algo:attack} with a randomly selected $h\in [n]$ and support $M\subset [n]\setminus\{h\}$ of size $m=\Omega(k^2)$ and $r=O(\gamma^2 \alpha^{-2} k^2\log^2 k)$ queries. Then with constant probability one of the following holds. Either the error rate of responses  $s^{(t)}$ for the $(1,\alpha)$-gap problem on the input norm $\|\boldsymbol{v^{(t)}}\|_2$ (see \cref{gap:def}) exceeded some constant $\delta(\alpha)>0$ or the vector $\boldsymbol{z}^{(\mathrm{adv})}$
    is a $\gamma$-adversarial noise  vector (see \cref{adversarial:def}) for $A$, signal $h$ and noise support $M$.
\end{restatable}
We define a vector as $\gamma$-\emph{adversarial} if, under the query distribution specified by $h$ and $M$, the optimal estimator returns a value that deviates from the true norm by a multiplicative factor of at least $\gamma$. 

\subsection{Density and parameter setting}

\begin{definition} [Signal density and parameters] \label{signaldensity:def}
  The distribution $W$ we use for $w$ is parametrized by 
$a< 1< 1+\alpha < b$ and
has density function:
\[
C \;=\; \frac{2}{\,b-a+\alpha\,},\qquad
\nu(w)\;=\;
\begin{cases}
0, & w<a \;\text{or}\; w>b,\\[6pt]
C\,\dfrac{w-a}{\,1-a\,}, & a \le w \le 1,\\[10pt]
C, & 1 < w < 1+\alpha,\\[10pt]
C\,\dfrac{\,b-w\,}{\,b-1-\alpha\,}, & 1+\alpha \le w \le b.
\end{cases}
\]

We set the parameters as follows according to the gap $\alpha>0$ and the error rate $\delta>0$ allowed for the responder. We use $\delta>0$ that satisfies  $\delta/\log(1/\delta)=O(\alpha^2)$ and $c$ that is a small constant (that does not depend on $\alpha$ and is selected according to other constants). 
\begin{align*}
    a &= 1- 10\alpha/c \\ 
    b &= 1+\alpha+ 10\alpha/c  
\end{align*}
\end{definition}

Observe from our settings that each of the intervals $[a,1]$ and $[1+\alpha,b]$ has at least a constant fraction of the probability mass.

\section{Analysis of the attack} \label{analysis:sec}

This section presents the key components and outlines the proof of \cref{attackefficacy:thm}, with full details deferred to the appendix.

We introduce notation for the query and noise distributions.
For indices $M\subset [n]$, let $\mathcal{N}_{n,M}$ be the distribution over $\boldsymbol{u}\in \mathbb{R}^n$ in which $u_{[n]\setminus M} = \boldsymbol{0}$ and the coordinates indexed by $M$ are sampled from the $m=|M|$-dimensional Gaussian distribution $u_{M} \sim  \mathcal{N}(0, \frac{1}{m}I_{m})$. Note that $\mathcal{N}_{n,M}$ is the distribution of noise vectors selected in \cref{algo:attack}.

For $h\in[n]$, $M\subset [n]\setminus\{h\}$, and noise scale $c$ let 
\begin{equation} \label{qdist:eq}
    F_{h,M,c}[w] := w \boldsymbol{e}_h  + c\mathcal{N}_{n,M} 
\end{equation}
be the distribution of vectors formed by adding a scaled noise vector sampled from $\mathcal{N}_{n,M}$ to a \emph{signal} $w \boldsymbol{e}_h$, where $\boldsymbol{e}_h \in \mathbb{R}^n$ is the standard basis vector at index $h$. $F_{h,M,c}[w]$ is the distribution of query vectors selected in \cref{algo:attack} for signal value $w$.

Our analysis is in terms of \emph{signal estimation}. In order to facilitate it, we establish that a correct norm gap output yields a correct signal gap output with similar parameters:

\begin{restatable}[name=Norm gap to signal gap]{lemma}{normtosignal}
\label{lem:norm-to-signal}
    With the choice of parameters for our attack and $m=\Omega((k+r) \log ((k+r)/\delta))$, with probability close to $1$, 
    a correct $(1-c^2, 1.1\alpha)$-norm gap output implies a correct $(1, \alpha)$-signal gap output on all queries.
\end{restatable}
The proof is included in \cref{normtosignalproof:sec}. The norm gap in the statement of \cref{attackefficacy:thm} is with parameters $(1,\alpha)$. To reduce clutter, we treat it in the sequel as signal gap of $(1,\alpha)$.

\ignore{
The query responder applies a norm gap estimator to the sketch $A\boldsymbol{v}$

\begin{definition} [Signal Gap Estimator] \label{correct:def}
  A map $\psi:\mathbb{R}^k\to \{-1,+1\}$ is  $(1,\alpha,\delta)$-correct signal gap estimator: 
  \begin{align*}
\text{for $w \in (c,1]$, } &        \E_{\boldsymbol{u}\sim \mathcal{N}_{n,M}(0,\sigma^2)}\left[-\psi\left(A(w\cdot \boldsymbol{e}_h +  \boldsymbol{u})\right) \right]\leq\delta\\
\text{for $w \geq 1+\alpha$, } &
        \E_{\boldsymbol{u}\sim \mathcal{N}_{n,M}(0,\sigma^2)}\left[\psi\left(A(w\cdot \boldsymbol{e}_h +  \boldsymbol{u})\right) \right]\geq 1- \delta
    \end{align*}
\end{definition}

\begin{claim} [Correct norm gap estimator yields a correct signal gap estimator]
    An $\ell_2$ norm estimator $M$ such that for any $w$,
    \[
   \E_{\boldsymbol{u}\sim \mathcal{N}_{n,M}(0,\sigma^2)}\left[
   \frac{\left|M\left(A(w\cdot \boldsymbol{e}_h +  \boldsymbol{u})\right) - \|\boldsymbol{e}_h +  \boldsymbol{u}\|_2^2\right|}{ \|\boldsymbol{e}_h +  \boldsymbol{u}\|_2^2} \leq \alpha/4 
   \right]\geq 1- \delta 
    \]
    yields an $(1,\alpha,\delta)$-correct estimator as in \cref{correct:def}
\end{claim}
\begin{proof}
    The contribution of $\|\boldsymbol{u}\|^2_2$ to the norm is tightly concentrated around $c^2$. Since $w^2\geq a^2> c^2$, a relative error of $\alpha/4$ on the squared norm implies a relative error of at most $\alpha/2$ on $w^2$. We then return $0$ if the estimate is below $1+\alpha/2$ and $1$ otherwise.
\end{proof}
}


\ignore{
\begin{definition} [Signal Gap Estimator] \label{signalcorrect:def}
For $A\in\mathbb{R}^{k\times n}$, we say that $\psi:\mathbb{R}^k\to \{-1,1\}$ 
  is an $(\alpha,\delta)$-correct signal gap estimator for $(A,h,M,\sigma^2)$ if for all $w>0$:
  {\small
  \begin{align*}
      \E_{\boldsymbol{v}\sim F_{h,M,\sigma^2}[w]}\left[\mathbf{1}\{\psi\left(A\boldsymbol{v}\right)=1\} \cdot \mathbf{1}\{w\leq 1\} + \mathbf{1}\{\psi\left(A\boldsymbol{v}\right)= -1\}\left(A\boldsymbol{v}\right))\cdot \mathbf{1}\{w \geq 1+\alpha\} \right]\leq\delta
    \end{align*}}
    That is, the probability over the noise distribution of a correct $(1,\alpha)$- signal gap response (\cref{gap:def}) is at least $1-\delta$.
\end{definition}
}

\subsection{Signal estimation and the optimal estimator}

For a fixed sketching matrix $A$, $h\in [n]$, and noise support $M$, we express the optimal estimator on the signal $w$ from a sketch $A\boldsymbol{v}$ when $\boldsymbol{v}\sim F_{h,M,c}[w]$.
For this purpose we may assume that the distributions (and $A$, $h$, and $M$) are given to the responder.
Note that if it holds that
$A_{\bullet h} = \boldsymbol{0}$, then the sketch carries no information on the signal $w$.
When this is not the case, we can express the optimal unbiased estimator.
The proof is included in \cref{signalsufficient:sec}.

\begin{restatable}[name=Estimator for the Signal]{lemma}{signalsufficient}
\label{signalsufficient:lemma}
Fix $h \in [n]$, a noise support set $M \subset [n] \setminus \{h\}$, and a noise scale factor $c$. 
Consider the distributions $F_{h,M,c}[w]$ parametrized by $w$.

If the column $A_{\bullet h}$ is nonzero, then there exists an unbiased, complete, and sufficient statistic $T_{h,M}: \mathbb{R}^k \to \mathbb{R}$ for the signal $w$ based on the sketch $A \boldsymbol{v}$.

Furthermore, the deviation of this estimator from its mean, defined as 
\[
\Delta_{h,M}(c A \boldsymbol{u}) := T_{h,M}(A \boldsymbol{v}) - w,
\]
depends only on the sketch of the noise $\boldsymbol{u} \sim \mathcal{N}_{n,M}$, and is distributed
as a Gaussian random variable $\mathcal{N}(0, c^2\sigma_T^2(h,M)$.
\end{restatable}
The estimator $T_{h,M}(A \boldsymbol{v})$ also minimizes the
mean squared error (MSE)~\citep{lehmann1998point}.

We define an adversarial noise vector to be one that causes
a large deviation in this optimal estimator:
\begin{definition} [$\gamma$ adversarial noise]\label{adversarial:def}
  A unit vector $\boldsymbol{u}$ with support $M$ is $\gamma$-adversarial for $A$, $h$, $M$ if $|\Delta_{h,M}(\boldsymbol{u})| > \gamma$
\end{definition}
Adversity of $\gamma$ means that the value is $\gamma/(c\sigma_T)$ standard deviations off. We will see that the attack size needed for certain adversity depends on $\sigma_T$.

\subsection{Lower bounding the error}
Since $T_{h,M}(A \boldsymbol{v})$ minimizes the MSE for
estimating the signal $w$ from the sketch, it implies a 
lower bound on the error that applies for any query responder on the query distribution of \cref{algo:attack}:
\begin{corollary}
    The mean squared error (MSE) on \emph{any} estimator, even one that is tailored to $h$, $M$, and $c$, on queries of the form $\boldsymbol{v}\sim F_{h,M,c}[w]$ when $w\sim W$ 
    is $\Omega(c^2\sigma_T(h,M)^2)$. 
\end{corollary}

We next establish that for any sketching matrix $A\in \mathbb{R}^{k\times n}$, a random choice of signal index $h\in [n]$ and noise support (of size $M$ that is slightly superlinear in $k$), it is likely that either column $h$ is all zeros (and hence any estimator must fail with constant probability) or
$\sigma^2_T(h,M)=\tilde{\Omega}(1/k)$. The proof is in \cref{sigmalower:sec}.

\begin{restatable}[Lower Bound on Error]{lemma}{lowerbounderror}
\label{sigmaTlower:lemma}
Let $A \in \mathbb{R}^{k \times (m+1)}$ be a matrix with $m \geq 20 k\log^2 k$.  
Then, for at least $0.9$ fraction of columns $h \in [m+1]$, it holds that either 
$A_{\bullet h}$ is all zeros or
$\sigma_T^2(h, M= [m+1] \setminus \{h\}) = \Omega\left(\frac{1}{k\log k}\right)$.
\end{restatable}

Recall that when the input vectors are $k$-sparse, exact recovery of the norm is possible, and hence there is no estimation error. Therefore, there is potential vulnerability only when the sparsity of the query vectors exceeds $k$ and thus our slightly super linear sparsity is necessary.

As a corollary of \cref{sigmaTlower:lemma}, we obtain a lower bound on the relative error guarantees of any estimator: 
\begin{corollary}
    Consider a fixed sketching matrix 
    Let $A\in\mathbb{R}^{k\times n}$ be a sketching matrix with $n = \Omega(k\log^2 k)$. Consider the distribution of queries that samples $h$ and $M$ (of size $m = \Omega(k\log^2 k)$) randomly from $[n]$ and then sample a query from $F_{h,M,c}[w]$ where $w\sim W$ (and the density of $W$ is at least a constant in range of size $>100/\sqrt{k}$). Then with at least constant probability, the MSE of any sketch-based signal (and hence norm) estimator is $\Omega(1/(k\log k)$. 
\end{corollary}

\begin{remark} [Error for JL matrices]
    Sampled matrices from the JL distributions meet this upper bound: For fixed noise support $M$ of size $m>k\log k$, with high probability over the sampling of $A$, for all $h\in [n]$, $\sigma^2_T(h,M) =O(1/k)$ (the property needed for the sampled $A$ is that the projected rows $A_{iM}$ are close to orthogonal).    
\end{remark}


\subsection{The gain lemma}

We quantify the expected progress (towards an adversarial vector) in each step of the attack in terms of $\sigma_T^2(h,M)$. What we show is that when the estimator has a low error rate, then
queries for which the noise component $\boldsymbol{u}$ has a higher deviation $\Delta_{h,M}(A\boldsymbol{u})$ are more likely to have $s^{(t)}=1$. Specifically, even though 
$\E[\Delta_{h,M}(A\boldsymbol{u})]=0$ (as the deviation has distribution $\mathcal{N}(0,\sigma^2_T)$), the response $s^{(t)}$ is correlated with it and we will show that
$\E[\Delta_{h,M}(s^{(t)} A\boldsymbol{u})] \propto \sigma_T^2(h,M)$. 

\begin{definition} [Error rate of sketch-based signal gap estimator] \label{signalcorrect:def}
An estimator $\psi:\mathbb{R}^k\to \mathcal{P}\{-1,1\}$ is a map from a sketch to a probability distribution on $\{-1,1\}$. For a sketching map $A\in\mathbb{R}^{k\times n}$, $h,M,c$ and signal distribution $W$,
the \emph{error rate} of $\psi$ for the $\alpha$-signal gap problem (\cref{gap:def}) is the probability over the query distribution of an incorrect $(1,\alpha)$- signal gap response:
  {\small
  \begin{align*} \Err(\psi) :=
      \E_{w\sim W, \boldsymbol{v}\sim F_{h,M,c}[w]}\left[\Pr[\psi\left(A\boldsymbol{v}\right)=1] \cdot \mathbf{1}\{w\leq 1\} + \Pr[\psi\left(A\boldsymbol{v}\right)= -1]\cdot \mathbf{1}\{w \geq 1+\alpha\} \right]\ .
    \end{align*}}
\end{definition}

If $A_{\bullet h}=\boldsymbol{0}$ then any estimator would have at least a constant error rate. 
Additionally, from standard tail bounds, for error rate that is $O(\delta)$, it must hold that $\alpha> c\, \Omega(\sqrt{\log(1/\delta)}) \sigma_T$ and therefore $\sigma_T < \alpha/(c \sqrt{\log(1/\delta)} $.

We therefore consider the case where $A_{\bullet h}\neq \boldsymbol{0}$, and therefore unbiased $T_{h,M}$ exists and $\sigma^2_T(h,M)$ is well defined, and assume that
$\sigma_T < \alpha/(c \sqrt{\log(1/\delta)} $.
We quantify the per-step expected gain (see \cref{gainproof:sec} for the proof):

\begin{restatable}[Gain Lemma]{lemma}{gainlemma}
\label{gain:lemma}
If $\Err[\psi] < \delta$ and attack parameters are as in~\cref{signaldensity:def}, and $\sigma_T < \alpha/(c\sqrt{\log(1/\delta)})$ then
    \[
    \E_{\boldsymbol{v}}\left[\psi(A\boldsymbol{v})(T_{h,M}(A\boldsymbol{v})-w)\right] = \Omega(\alpha^{-1} c^3\sigma_T(h,M)^2)
    \]
\end{restatable}

\subsection{Proof of \nameref{attackefficacy:thm} Theorem}


\begin{proof}[Proof of \cref{attackefficacy:thm}]
The proof idea is to bound the deviation and the norm of the sum $\sum_{t\in[r]} s^{(t)}\boldsymbol{z}^{(t)}$ and show that the deviation increases faster than the norm.

We first bound norm. We show that if the noise support is sufficiently wide with $m> r^2$, then for \emph{any} choice of $s^{(t)}$, the norm of $\sum_{t\in[r]} s^{(t)}\boldsymbol{z}^{(t)}$ can not be much higher than that of the sum $\sum_{t\in[r]} \boldsymbol{z}^{(t)}$ of independent Gaussians:
\begin{align}
\left\| \sum_{t\in[r]} s^{(t)}\boldsymbol{z}^{(t)} \right\|_2 = O(\sqrt{r})\label{sumnorm:eq}
\end{align}
This follows as an immediate corollary of \cref{lem:signed-gauss-upper}.

We now consider the deviation of the sum. 
From \cref{gain:lemma} we obtain that for each $t$,
$\E_{\boldsymbol{v}}[s^{(t)}\Delta_{h,M}(A\boldsymbol{z}^{(t)})] = \Omega(\alpha^{-1}\, \sigma^2_T(h,M))$.

From concentration of $\Delta_{h,M}(A\boldsymbol{z}^{(t)})]$ (that are independent $\mathcal{N}(0,\sigma_T^2)$) we obtain that with high probability
\begin{align} 
    \sum_{i=1}^r s^{(t)}\Delta_{h,M}(A\boldsymbol{z}^{(t)}) = \Omega(r\,  \alpha^{-1}\, \sigma^2_T(h,M))\ .\label{sumbias:eq}
\end{align}

From \cref{sigmaTlower:lemma},
for any sketching matrix $A\in \mathbb{R}^{k\times n}$, when we sample $h$, and $M\subset [n]$ of size $m = \Omega(k\log^2 k)$, then 
with constant probability we have ${\sigma}_T^2(h,M) = \Omega(\frac{1}{k\log k}))$.

Combining with \eqref{sumbias:eq} we obtain that 
\begin{align} 
    \sum_{i=1}^r s^{(t)}\Delta_{h,M}(A\boldsymbol{z}^{(t)}) = \Omega\left(\frac{r\alpha^{-1}}{k\log k}\right)\ . \label{sumbiask:eq}
\end{align}

Combining \eqref{sumnorm:eq}  and \eqref{sumbiask:eq} we obtain that with constant probability
\begin{align*}
    \Delta_{h,M}(A\boldsymbol{z}^{(\mathrm{adv})}) = \Omega\left(\frac{\sqrt{r}\alpha^{-1}}{k\log k}\right)\ . 
\end{align*}

For the deviation to exceed $\gamma$, solving for
$\frac{\sqrt{r}\alpha^{-1}}{k\log k} > \gamma$, we obtain
$r= O(\gamma^2\, \alpha^{-2}\, k^2\log^2 k)$.

This concludes the proof.
\end{proof}

\section{Empirical study} \label{empirical:sec}

We implemented Algorithm~\ref{algo:attack} and evaluated its effectiveness on two
families of sketching matrices:
Gaussian Johnson--Lindenstrauss (JL) transforms~\citep{johnson1984extensions} with
$A \in \mathbb{R}^{k \times n}$
with i.i.d.\ entries $A_{ij} \sim \mathcal{N}(0,1/k)$
and AMS sketches~\citep{ams99} with
$A \in \{\pm 1\}^{k \times n}$ consisting of i.i.d.\ Rademacher entries
($A_{ij} = \pm 1$ with equal probability).
Our evaluation used different configurations of  sketch dimension $k$ and ambient dimension $n$ (effectively, the noise support).
For each configuration we applied the attack against the corresponding
\emph{standard} norm estimator and its \emph{robustified} variant, and
measured how rapidly the adversarial bias grows.

\paragraph{Standard Estimators.}
For Gaussian JL matrices, the \emph{standard} norm estimator is simply the
sketch norm, which coincides with the minimum-variance unbiased estimator
for $\|v\|_2^2$ under Gaussian projections:
\[
\widehat{\|v\|_2^2} := \|A \boldsymbol{v}\|_2^2.
\]
For AMS sketches, we implemented the
\emph{median-of-means} (MoM) estimator.
Specifically, we partition the $k$ rows of $A$ into
$g = \max\{5,\lfloor \sqrt{k} \rfloor \}$ groups of equal size
$b = \lceil k / g \rceil$, compute the mean of $y_i^{2}$ within each
group, and take the median of the $g$ group means:
\[
\widehat{\|v\|_2^2}
:= \operatorname{median}\!\Bigl(
  \frac{1}{b}\!\sum_{i \in G_1} y_i^{2},\;
  \frac{1}{b}\!\sum_{i \in G_2} y_i^{2},\;
  \dots,\;
  \frac{1}{b}\!\sum_{i \in G_g} y_i^{2}
\Bigr).
\]
Our final estimate is
$\widehat{\|v\|_2} := \sqrt{\widehat{\|v\|_2^2}}$.

\paragraph{Simplified attack} We implemented a simplified attack that sufficed for the special case of JL/AMS matrices
(distribution that is invariant under column permutations) and query responders that are not adaptive and not tailored to the input distribution). Our query vectors 
are sampled i.i.d.\ from $F_{n,[n-1],1}[w=1]$ (see~\cref{qdist:eq}), that is, use a
a fixed signal value $w=1$ and fixed $h=n$ and $M=[n-1]$. Observe that for our query vectors, the norm $\|v\|_2$ is concentrated around $\sqrt{w^2+1}=\sqrt{2}$.

\paragraph{Robust Estimators.}
The robust variants are parameterized by a noise scale $\sigma$ and add
Gaussian noise to the squared-norm standard estimate before taking the square root:
\[
\hat{s}_{\sigma}(\boldsymbol{v})
:= \sqrt{\widehat{\|v\|_2^{2}} + \mathcal{N}(0,\sigma^{2})}.
\]
In our attack implementation the responder outputs $s=1$ when
$\hat{s}_{\sigma} \geq \sqrt{w^{2}+1}$ and $s=-1$ otherwise.
\footnote{Our simplified attack would fail against a fully strategic
responder that is tuned to the input distribution and always returns
the deterministic value $\sqrt{w^{2}+1} = \sqrt{2}$.
Such a responder leaks no information about the sketching matrix and
is correct with high probability on all queries.}

On freshly sampled (non-adaptive) inputs, the estimator has standard deviation $1/\sqrt{k}$ in the JL case, whereas for the robustified
estimator the standard deviation increases to 
$\sqrt{\frac{1}{k} + \sigma^{2}}.$ Therefore, the non-adaptive accuracy decreases with the robustness noise $\sigma$.


\paragraph{Experiments}
For each $(k,n)$ configuration of sketch size $k$ and support size $n$, we performed $\textrm{rep}=20$ trials. For each trial we sample a fresh sketching matrix $A \in \mathbb{R}^{k \times n}$. We sample query vectors as described and apply estimators with different values of the robustness noise $\sigma\in \{0,0.1,0.2,0.4\}$ to the same query vectors (observe that $\sigma=0$ is the standard estimator). 
We study the effectiveness of our attack 
by tracking the ratio 
$\hat{s}_\sigma(\boldsymbol{z}^{(\mathrm{adv})})/\boldsymbol{z}^{(\mathrm{adv})}$, where $\hat{s}_\sigma$ is the (respective) standard estimator that is applied to the sketch and  $\boldsymbol{z}^{(\mathrm{adv})}$ is the  adversarial vector.

\paragraph{Results}
\cref{JLattack:fig} and \cref{AMSattack:fig} report, for different configurations with $k\in\{100,250,1000\}$, the ratio $\hat{s}_\sigma(\boldsymbol{z}^{(\mathrm{adv})})/\boldsymbol{z}^{(\mathrm{adv})}$ as a function of the number of attack queries (the $x$-axis).

We observe that the attack is effective even in
configurations with small support size of $n/k \in [2,5]$. The standard and lower-$\sigma$ robust estimators are initially more accurate but also accumulate bias faster and incur a much higher error as the attack progresses. 
A larger sketch size $k$ results in higher robustness (slower increase in the bias) for the same noise $\sigma$. The plots show 
a quadratic pattern where the bias induced by the attack increases like square root of the number of queries. Results for additional configurations for the same $k$ are reported in \cref{JLattacksweepn:fig}.
We observe that the effectiveness of the attack tends to increase with a larger support size $n$.


\paragraph{Discussion.}
Our empirical results demonstrate that the attack is substantially more
effective in practice than our current analysis predicts.
In particular, we observe strong empirical success even when the noise
support size is very small—on the order of $O(k)$—whereas our analysis
requires $n = \Omega(k^{2})$.
The attack is also effective with relatively few queries, again beyond
what is guaranteed by our theoretical bounds.

Moreover, our analysis only establishes that the attack product
compromises the \emph{optimal} estimator for the sketching matrix,
which is a linear statistic of the sketch.
For JL sketches, the standard estimator is very close (up to row
orthogonality) to this optimal estimator, so our theory is largely
predictive.
In contrast, for AMS sketches we employ a median-of-means estimator,
which is non-linear, yet the attack empirically compromises it with
comparable efficiency.
This suggests that the vulnerability extends beyond linear estimators
and may be more fundamental than our current proofs capture.

\begin{figure}[h]  
    \centering  
    \includegraphics[width=0.30\textwidth]{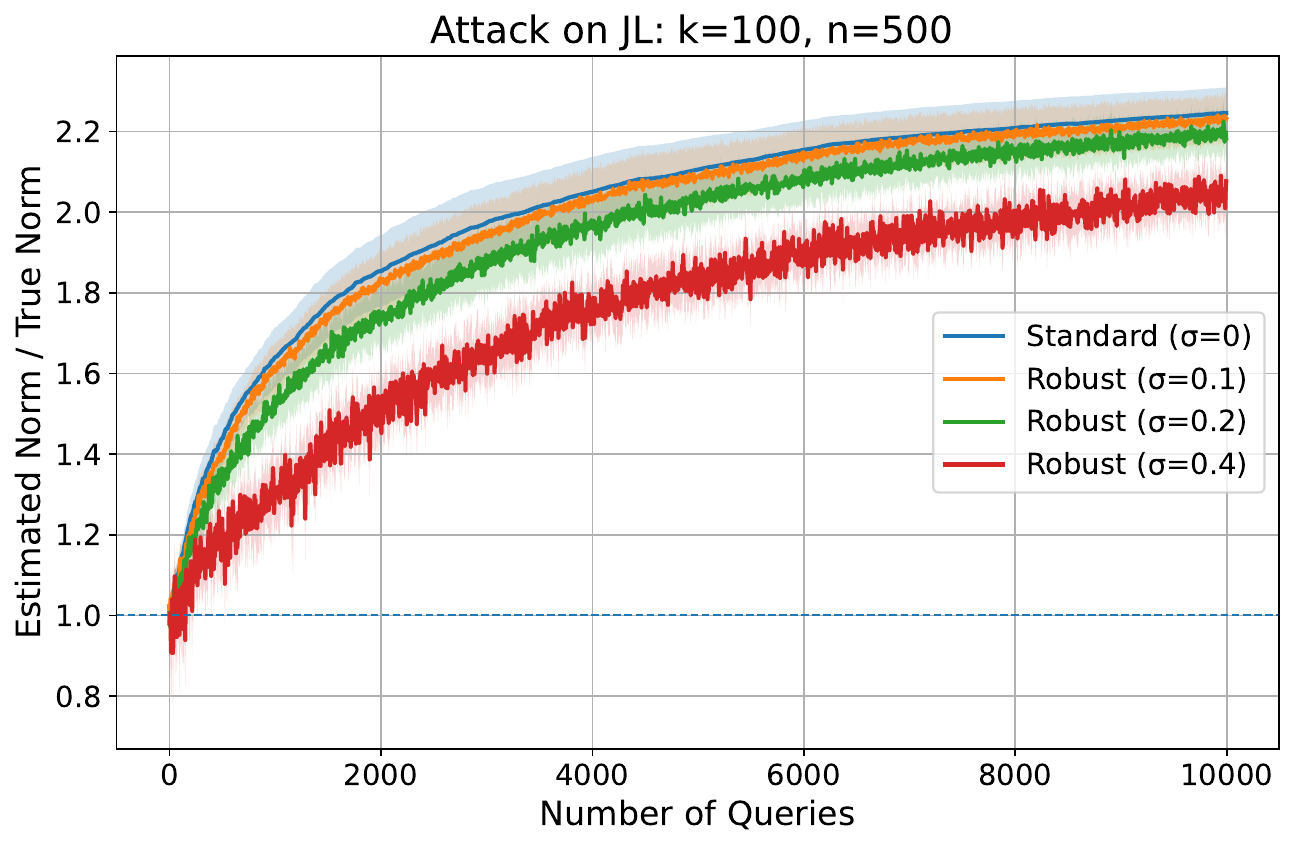}
    \includegraphics[width=0.30\textwidth]{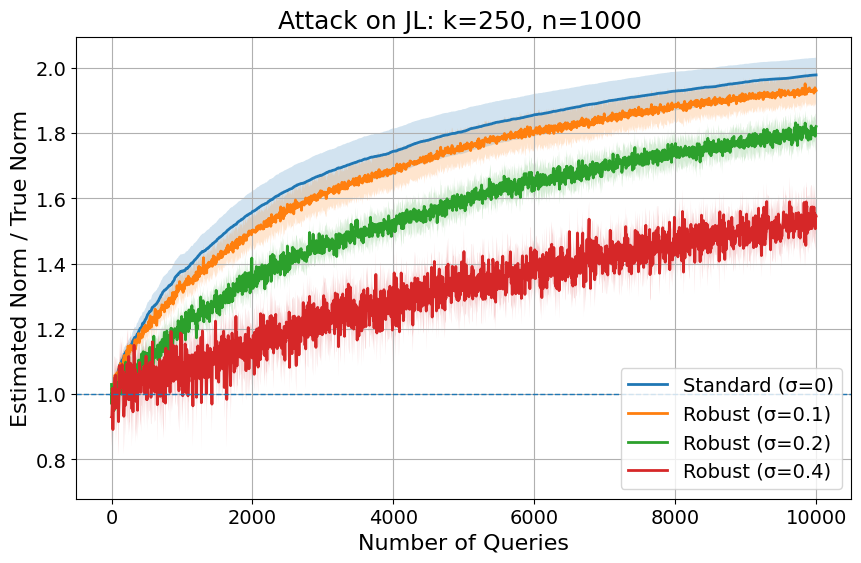}
    \includegraphics[width=0.30\textwidth]{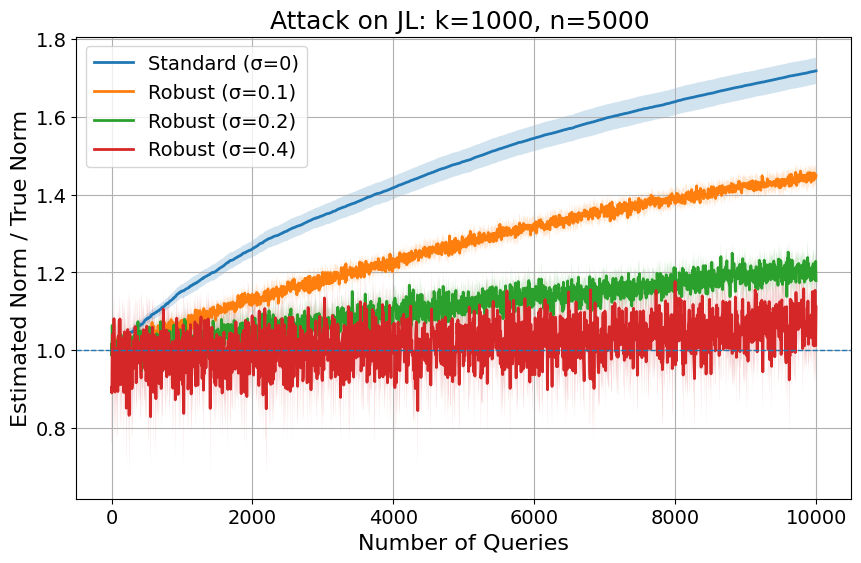}
    \caption{Attack on JL with $(k,n)\in\{(100,500), (250,1000), (1000,5000)\}$. Standard and robust estimators. Mean ratio of estimate to actual norm with 95\% confidence intervals over 20 repetitions.
\label{JLattack:fig}}
\end{figure}

\begin{figure}[h]  
    \centering  
    \includegraphics[width=0.30\textwidth]{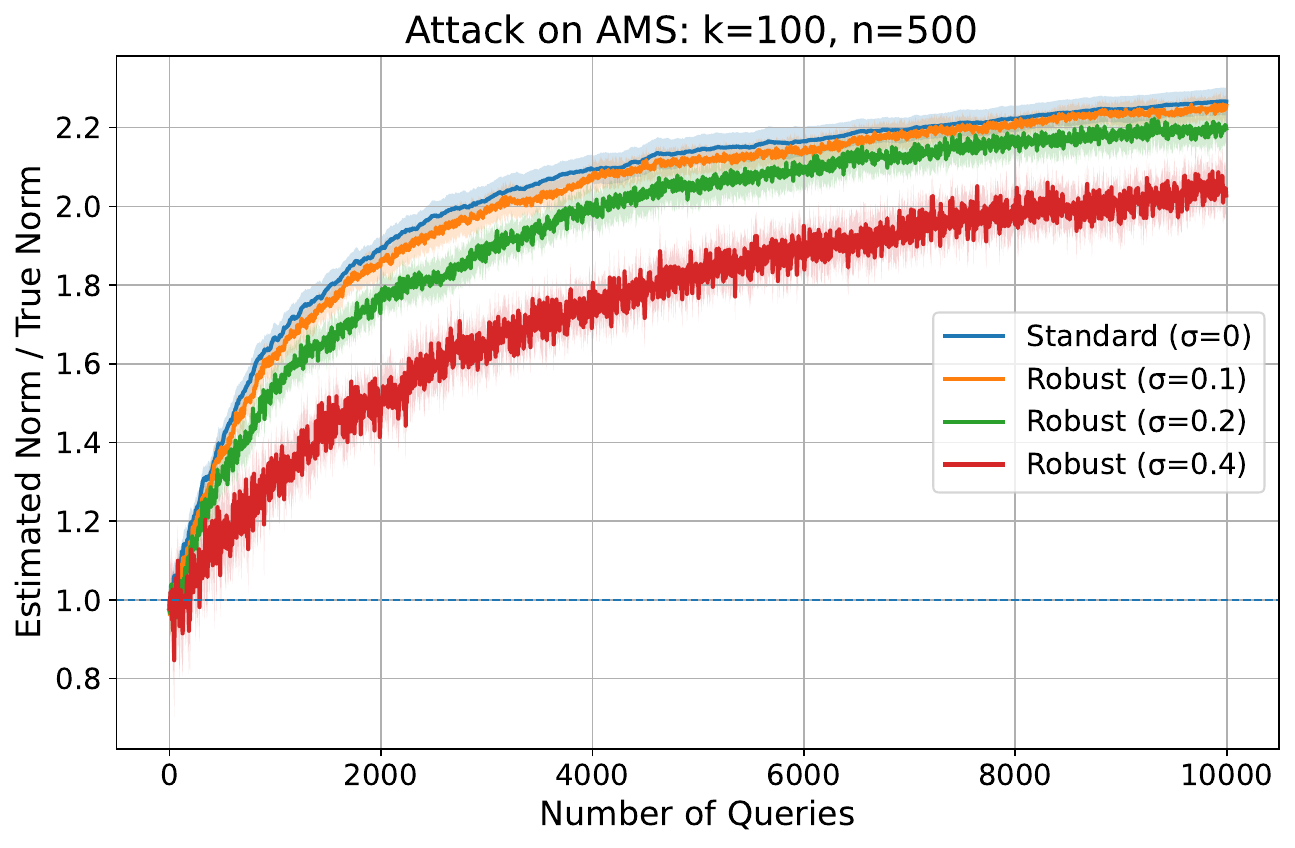}
    \includegraphics[width=0.30\textwidth]{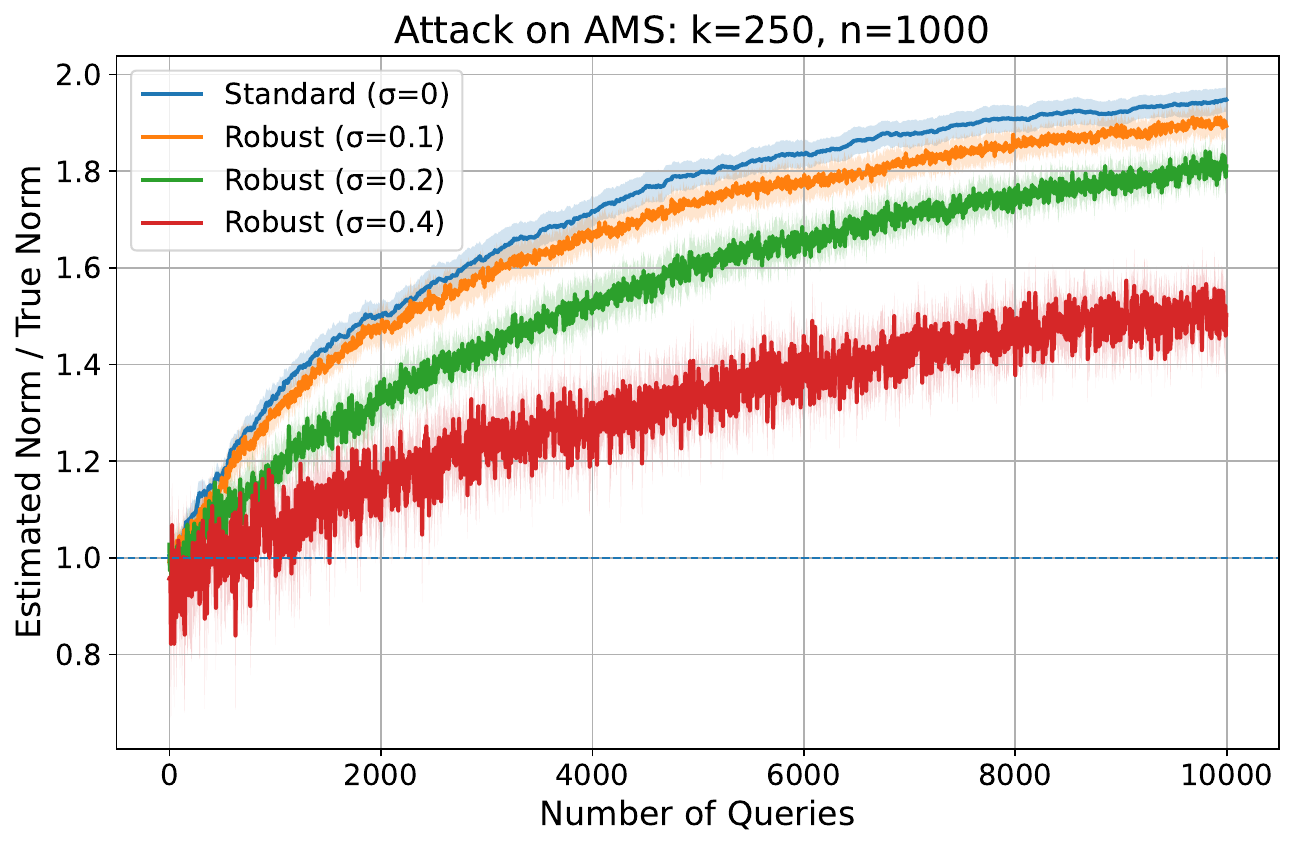}
    \includegraphics[width=0.30\textwidth]{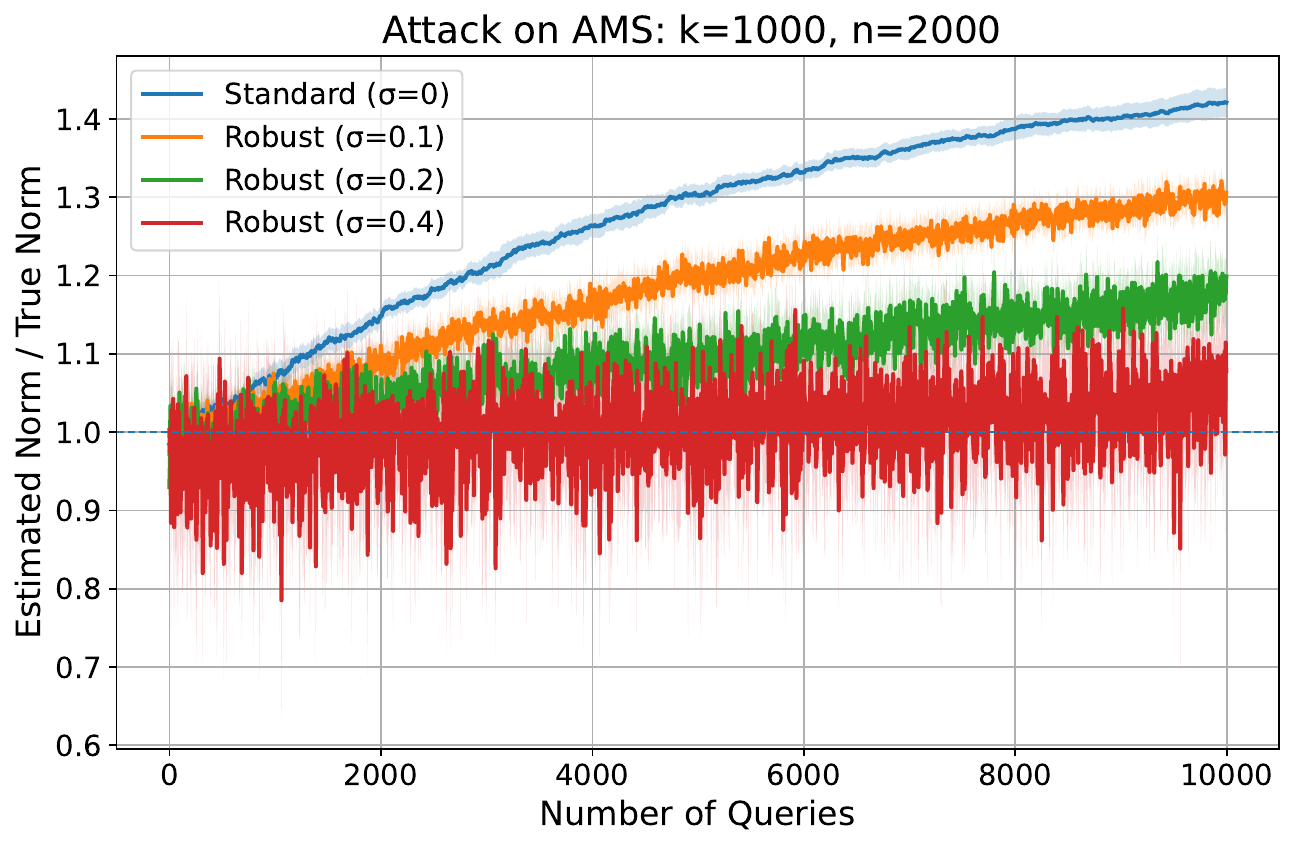}
    \caption{Attack on AMS with $(k,n)\in\{(100,500), (250,1000), (1000,2000)\}$. Standard and robust estimators. Mean ratio of estimate to actual norm with 95\% confidence intervals over 20 repetitions.
\label{AMSattack:fig}}
\end{figure}

\begin{figure}[h]  
    \centering  
    \includegraphics[width=0.30\textwidth]{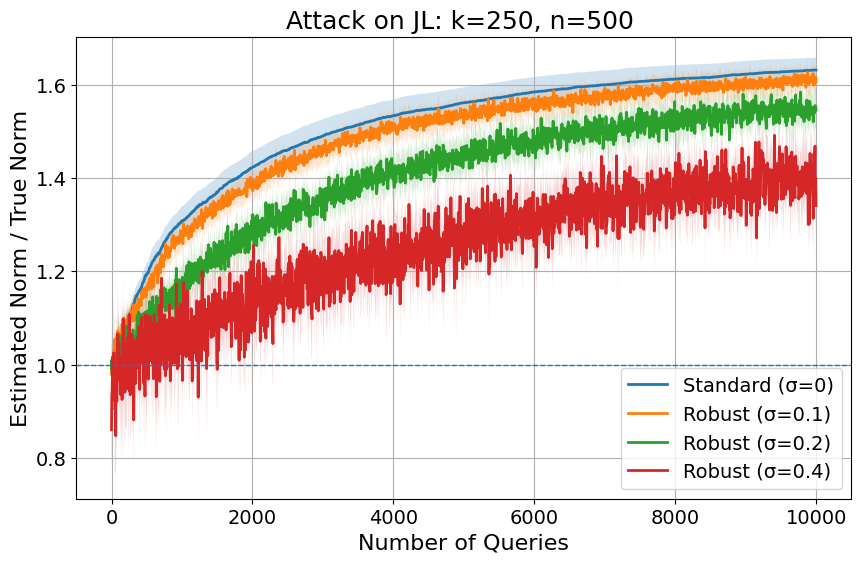}
    \includegraphics[width=0.30\textwidth]{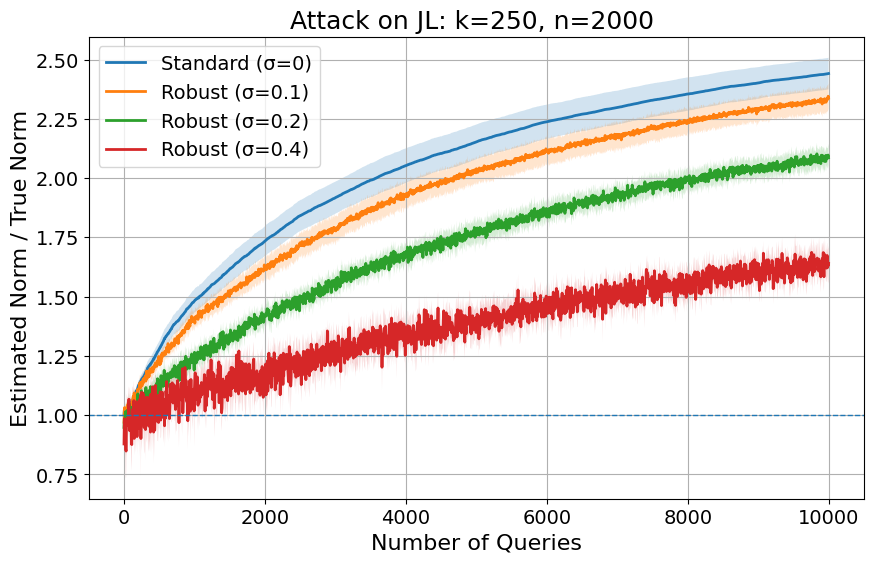}
    \includegraphics[width=0.30\textwidth]{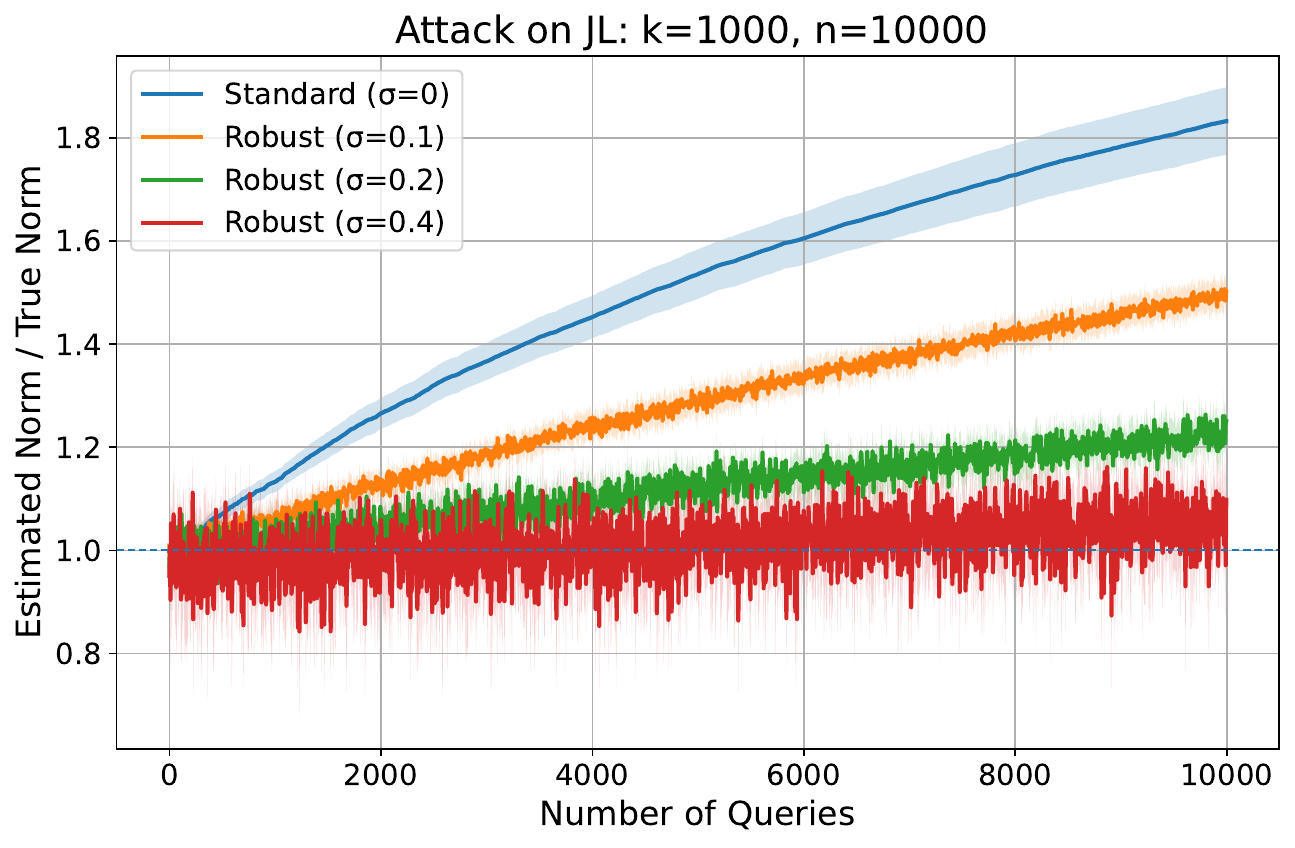}
    \caption{Attack on JL with $(k,n)\in \{(250,500), (250,2000), (1000,10000)\} $. Mean ratio of estimate to actual norm with 95\% confidence intervals over 20 repetitions.
\label{JLattacksweepn:fig}}
\end{figure}

\ignore{
\subsection{image classifiers}

  Single-batch or adaptive attacks against an image classifier (piggy back on published Github experiments).

  Baseline: Use the logit prediction weights. The usual attack (perhaps in a single batch). Show that it works.

  Robust estimator: We can add noise to the logits or prediction probabilities. But alternatively we can  just add noise to the image before classification. We can show that this works.

  Show that the full black-box still works (only of course takes longer).
  Argue that privacy noise results in higher protection (but less accuracy perhaps) than well-studied robust smooth classifiers in logit-aware settings. This explains the protection of smooth classifiers and shows its inherent limitations.

 Smooth classifiers and robustness:
 
The output is averaged over a noise ball of the input (few samples of it).
$\E_{\boldsymbol{u}\sim \mathcal{N}_{n,M}(0,\sigma^2)}\|A(\boldsymbol{v}+ \boldsymbol{u})\|_2$. This is "inspired" by proposed robustification for attacks on images. 

Note this will not help for $\ell_2$ norm with linear sketches. The squared norms it is just a shift (the norm of the sketch shifts by $\|A\|_F^2$. For proper norm it is roughly 
$\sqrt{\|A\boldsymbol{v}\|_2^2 + \|A\|_F^2}$ (second term scaled by the noise scale). So this does not help with robustness. 

But with images it smoothing appears to help. There are two components of the vulnerability. 

Experiment. For input noisy image. Return average of logits from $\ell$ noise vectors. Larger $\ell$ should be less robust in the linear case. 
}

\section{Conclusion} \label{conclusion:sec}
Our results further suggest that vulnerability to black-box attacks is an inherent consequence of dimensionality reduction. Nonetheless, it can be partly mitigated by sacrificing some non-adaptive accuracy in exchange for increased robustness, underscoring a fundamental trade-off. We close with several directions for further investigation.

\paragraph{Theory questions.}
We proved a tight, quadratic-size universal attack that succeeds against
any query responder, but our analysis only guarantees compromise of the
\emph{optimal} (linear) estimator for the sketch and attack distribution. 
Empirically we observed that our attack compromised also the non-linear median-of-means estimator for AMS. This suggests exploring broader applicability of our attack.  Another natural question is whether our technique can be extended
to other norms, as in \citet{Gribelyuk2025Lifting}.
A more ambitious goal is to construct a quadratic-size attack that generates a \emph{single distribution} capable of compromising \emph{every} query responder. Such an attack would necessarily require $\tilde{\Omega}(k)$ adaptive batches and could plausibly build on an enhanced version of our construction that incorporates the discovered adversarial directions into the input distribution~\citep{CNSSS:ArXiv2024}.
 %
Another open question concerns the required size of the noise support, which directly corresponds to the attacker's storage cost. Our current proof requires $n = \Omega(k^{2})$ because it relies on a general bound (Lemma~\ref{lem:signed-gauss-upper}) for arbitrary signed sums of Gaussian vectors. Empirically, however, we observe that our attack remains effective with support as small as $2k$, and we conjecture that $\tilde{O}(k)$ support should suffice.

\paragraph{Connections to Image Classifiers.}
Our findings on the vulnerability of linear sketches may shed light on
the phenomenon of adversarial examples in image classification, which
has been extensively studied in prior work~\citep{szegedy2013intriguing,
goodfellow2014explaining,athalye2018synthesizing,
papernot2017practical,MadryRobustness:ICLR2018,Salman:NEURIPS2019}. Many of these attacks exhibit a striking structural similarity to our setting: they construct adversarial examples by aggregating small, randomly oriented perturbations that consistently push the system's output in a particular direction. When normalized and accumulated, these perturbations yield a large, adversarial deviation that fools the model. This additive, alignment-based mechanism appears to be remarkably effective across domains.

Although image classifiers such as CNNs are highly nonlinear, they contain significant linear components, particularly in early layers. This observation raises a natural question: could adversarial susceptibility in image models arise, at least in part, from the same linearity-driven accumulation of aligned noise?
If so, it may be possible to improve robustness by introducing carefully designed randomness—such as injecting noise that is unknown to the attacker into internal representations or into the input.

\ignore{

\paragraph{Insights on image classifiers}
Our work allows us to gain insights into potential vulnerabilities for image classifiers demonstrated in prior works~\citep{szegedy2013intriguing,goodfellow2014explaining,athalye2018synthesizing,papernot2017practical,MadryRobustness:ICLR2018,Salman:NEURIPS2019}

These attacks have a common form to our proposed attack. This form appears surprisingly effective across domains: 
The attack vectors are formed by adding perturbations sampled from a noise distribution.
The adversarial input is constructing by linearly adding randomly sampled perturbations that resulted in tiny deviations in the same direction to the sketch-based response. The normalized sum appears to have a much larger deviation. 

The product of these attacks is an input that compromises the learned classifier, which in the learning process is optimized for the data distribution.

While the CNNs based image classifiers are not linear, they do have linear components. Can the explanation for the source of vulnerability be at least partly the same one? Can similar techniques trade accuracy for robustness?

\citep{MadryRobustness:ICLR2018} observed that robustness improves by the use of \emph{smooth classifiers}. A smooth classifier returns the average response over a Gaussian perturbation around the input. They explain the improvement by the smoothing of the decision boundaries. Interestingly, smooth classifiers do not add robustness with for $\ell_2$ norm estimations from linear sketches: as it essentially results in a fixed shift in the estimator.  
}

\ignore{
is that capacity (of the network) is critical for robustness. 
They explain in the sense of modeling more complex decision boundaries around ``noise balls'' of examples and propose robust classifiers via training methods that learn those boundaries by learning \emph{smoothed classifiers}. 

Interestingly, smooth classifiers do not add robustness for $\ell_2$ norm estimation from linear sketches: as it essentially results with a fixed shift in the estimator (an equivalent estimator). The theory, however, reveals another source of vulnerability that is inherent in the dimensionality reduction (that is mostly attributed to the architecture and capacity) and depends on the variance of the final prediction.

=========
}

\ignore{
\eccomment{the following should be part of discussion. Also can we design experiments that show benefits of noisy outputs?  Is it like averaging over the input noise balls and then sampling?}
An important takeaway in~\citep{MadryRobustness:ICLR2018} is that capacity (of the network) is critical for robustness. 
They explain in the sense of modeling more complex decision boundaries around ``noise balls'' of examples and propose robust classifiers via training methods that learn those boundaries by learning \emph{smoothed classifiers}. 

Interestingly, smooth classifiers do not add robustness for $\ell_2$ norm estimations from linear sketches: as it essentially results with a fixed shift in the estimator (an equivalent estimator). The theory, however, reveals another source of vulnerability that is inherent in the dimensionality reduction (that is mostly attributed to the architecture and capacity) and depends on the variance of the final prediction.

The insight we gain from the theory is that capacity (dimensionality) is also important even in linear models as it allows for a higher representation dimensionality. But to exploit this against attacks of the type proposed, there is a different technique needed which is adding noise to the final prediction. The potential gain in this case is quadratic in the internal dimensionality in terms of the number of queries needed to distort the result.

\eccomment{What might add to robustness of the classifier is sampling from an input noise ball (instead from the final). This mimics noisy output.}

These attacks were designed for when the final precise logit weights are revealed to the attacker. The theory suggests that to increase robustness, we need to protect these logits with noise. 
}

\newpage

\section*{Acknowledgments}

Edith Cohen was partially supported by Israel Science Foundation (grant 1156/23). 
Uri Stemmer was Partially supported by the Israel Science Foundation (grant 1419/24) and the Blavatnik Family foundation.

\bibliographystyle{plainnat}
\bibliography{references,robustHH}
\newpage
\appendix

\section{Norm estimation to signal estimation} \label{normtosignalproof:sec}

This section contains a restatement and proof of \cref{lem:norm-to-signal}:

\normtosignal*

\begin{proof}
The query vectors $\boldsymbol{v}= w \boldsymbol{e}_h  + c\boldsymbol{u}$ sampled in \cref{algo:attack}  are from distributions $F_{h,M,c}[w]$, for $w\in [a,b]$.
The squared norm is $\|\boldsymbol{v}\|^2_2 = w^2+ c^2\|\boldsymbol{u}\|^2_2$. 
We show that 
$\|\boldsymbol{u}\|^2_2$ is 
tightly concentrated around its expectation $1$:
\begin{claim} [Concentration of $\|\boldsymbol{u} \|^2_2$]
\[
    \Pr\left[ \left|\|\boldsymbol{u} \|^2_2 - 1\right|\geq \epsilon \right] \leq 2e^{- \frac{m}{2}(\epsilon^2/2)}
    \]
\end{claim}
\begin{proof}
    The sum of the squares of $m$ i.i.d.\ $\mathcal{N}(0,\sigma^2)$ has distribution $\sigma^2   \chi_m^2$ and expected value $m\sigma^2$. 
    Applying tail bounds on $\chi^2_m$ (Gaussian concentration of measure) we obtain that for $\epsilon\in (0,1)$
    \[
    \Pr\left[ \frac{\left|\|\boldsymbol{u} \|^2_2 - m\sigma^2\right|}{m\sigma^2}\geq \epsilon \right] \leq 2e^{- \frac{m}{2}(\epsilon^2/2)}\ .
    \] Substituting $\sigma^2=1/m$ we obtain the claim.
\end{proof}
It follows that  
if we choose $m=\Omega((k+r) \log ((k+r)/\delta))$ then with probability $1-\delta$, on all our queries, the squared norm of the noise is within 
$\|\boldsymbol{u}\|^2_2 \in (1\pm 1/(10\sqrt{k})$. 
Therefore, a correct $(1,\alpha)$-gap output on the norm yields a correct $(1,\alpha^2+2\alpha)$-gap on the squared norm
$\|\boldsymbol{v} \|^2_2$. This gives an $(1-c^2(1+1/(10\sqrt{k})),\alpha^2+2\alpha -c^2/(10\sqrt{k})$-gap output on the squared signal $w^2$.

We now note that we can assume $\alpha > 1/\sqrt{k}$ because otherwise,  any responder would be incorrect with constant probability.
Using our parameter setting of $\ell$ close to $1$ and fixed $c\ll 1$ we obtain
the claim in the statement of the lemma.
\end{proof}

\section{\nameref{signalsufficient:lemma}} \label{signalsufficient:sec}
This section contains a restatement and proof of \cref{signalsufficient:lemma}.

\signalsufficient*
\begin{proof}
 If the column $A_{\bullet h}$ is zero, then the sketch contains no information on the signal $w$ and unbiased estimation is not possible. Otherwise, our goal is to express the unbiased sufficient statistic $T_{h,M}(A\boldsymbol v)$ for the unknown scalar $w$.

Because row operations preserve the information in a sketch, there exists an invertible matrix $G\in \mathbb{R}^{k \times k}$ such that for $A' = G A$,
\begin{itemize}
  \item the transformed column $h$ has value $1$ in the first row and $0$ elsewhere: $A'_{1h}=1$ and $A'_{ih}=0$ for $i>1$; and
  \item every row $A'_{iM}$ for $i>1$ of the submatrix restricted to the noise coordinates, is either orthogonal to the first row $A'_{1M}$ or is the zero vector $\boldsymbol{0}$.
\end{itemize}

Since $G$ is invertible, $A'$ is equivalent to the original sketching matrix $A$ in that we can obtain the sketch $A'\boldsymbol{v}$ from
$A\boldsymbol{v}$ and vice versa.

We specify a sufficient statistics $T$ in terms of $A'\boldsymbol{v}$ (so that 
$T_{h,M}(A\boldsymbol{v})\equiv T(GA\boldsymbol{v})$).
Consider the distribution of the sketch $\boldsymbol{y} = A' \boldsymbol{v}=GA\boldsymbol{v}$ for
$\boldsymbol{v} = w \boldsymbol{e}_h + c\boldsymbol{u}$ where $\boldsymbol{u}\sim \mathcal{N}_{n,M}$.
We have $y_1 = w + c\sum_{j\in M} A'_{ij} u_j$, hence it has distribution
\begin{equation} \label{Tdensity:eq}
c \mathcal{N}(w,\sigma^2_T)\ , \text{ where  } \sigma^2_T := \| A'_{1M}\|^2_2 \sigma^2 = \frac{1}{m}\| A'_{1M}\|^2_2\ .
\end{equation}
From orthogonality of the rows $A'_{iM}$ restricted to $M$, the random variables $y_i$ for $i>1$ are independent of $A'_{1\bullet} \boldsymbol{u}$ and hence convey no information on $w$.  
The information on the signal $w$ in the sketch, and the unbiased sufficient statistic is therefore
$y_1=(GA\boldsymbol{v})_1$, which has distribution~\eqref{Tdensity:eq} and is an unbiased estimator of $w$.

To establish the claim for the additive error, note that
\[
\Delta_{h,M}\left(cA\boldsymbol{u}\right) =  T_{h,M}(A\boldsymbol{v})-w = (GA\boldsymbol{v})_1 -w = (cGA\boldsymbol{u})_1 \ .
\]
\end{proof}

\section{Proof of \nameref{sigmaTlower:lemma} Lemma} \label{sigmalower:sec}
This section contains a restatement and proof of \cref{sigmaTlower:lemma}.

\lowerbounderror*

We will use the following technical claims

\begin{definition} [Fragile Columns] \label{fragilecolumns:def}
    Let $A\in\mathbb{R}^{k\times m}$.  For each column $h\in[m]$, define 
\[
K_h \;=\;\{\,i\in[k]:A_{i h}\neq0\}, 
\qquad
b_{i h} \;=\;\bigl|\{\,j\in[m]:A^2_{i j}\ge A^2_{i h}\}\bigr|
\]
the  set $K_h$ of active rows and the dominated number $b_{ih}$ of each active row.
Let  $i\in K_h$
\[
c^{(h)}_1\le c^{(h)}_2\le\cdots\le c^{(h)}_{|K_h|}
\]
be the nondecreasing rearrangement of $\{b_{i h}:i\in K_h\}$).  
We say that a column $h$ is \emph{fragile}, if:
\[
\forall\,1\le i\le|K_h|:\quad
c^{(h)}_i \;\ge\;\frac{i\,m}{10 k\log_2 k}.
\]
\end{definition}
Note that zero columns are fragile by definition.

\begin{claim} [Most Columns are Fragile] \label{fragilecolumns:claim}
If $m>2 k\log_2 k$ then at least $0.9 m$ of the columns are fragile.
\end{claim}
\begin{proof}
We look at $\lfloor \log_2 k\rfloor$ ranges of $b_{ih}$ values where $R_1 = [1,\ldots,\frac{m}{10 k\log_2 k})$ and  $R_t = [2^{t-2},2^t)\frac{m}{10 k\log_2 k}$ for $t\geq 2$. 

We say a column $h$ is \emph{strong} for range $t$ if it has $2^{t-1}$  or more $b_{ih}$ values in $R_t$. That is, $|\{i : b_{ih}\in R_t\}|\geq 2^{t-1}$.
Clearly if a column $h$ is not fragile then it must be strong for some range $t$, but the converse may not hold.

We bound the number of distinct columns that are strong for at least one range $t$. This bounds the total number of non-fragile columns.

Consider range $t=1$. Each of the $k$ rows contributes the columns of its $\frac{m}{10 k\log_2 k})$ largest entries. So the total number of columns that can be strong for range $t=1$ is at most $\frac{m}{10 \log_2 k})$.

Now consider $t>1$. Each row has at most $2^{t-1} \frac{m}{10 k\log_2 k}$ distinct columns $h$ with $b_{ih}\in R_t$. But for a column to be strong for $t$ it has to participate in $2^{t-1}$ such rows.  So in total, the range contributes at most $\frac{m}{10 \log_2 k})$ strong columns.

Summing over all ranges, we obtain a bound of $m/10$ on the total number of columns that are strong for at least one range. This also bounds the number of non-fragile columns. This concludes the proof.
\end{proof}

\begin{claim} [Nearly to fully orthogonal rows]\label{nearlytofully:claim}
    Let $\boldsymbol{v}^{(1)},\dots,\boldsymbol{v}^{(k)}\in\mathbb{R}^n$ for $k\geq 1$ be linearly independent and satisfy
\[
\langle \boldsymbol{v}^{(i)},\boldsymbol{v}^{(j)}\rangle = 1 \quad(i\neq j),
\qquad
\|\boldsymbol{v}^{(i)}\|^2 >  \,i (2+\ln k)
\quad(i=1,\dots,k).
\]
Then there are orthogonal
$\boldsymbol{u}^{(1)},\ldots,\boldsymbol{u}^{(k)}$ such that $\boldsymbol{u}^{(1)}=\boldsymbol{v}^{(1)}$, and
\begin{align*}
    \|\boldsymbol{u}^{(i)}\|^2_2 >  \|\boldsymbol{v}^{(i)}\|^2_2 - 1 \geq i (1+\ln k)
\quad(i=2,\dots,k)\\
\boldsymbol{u}^{(i)} \text{ is an affine combination of } \boldsymbol{v}^{(1)},\ldots,\boldsymbol{v}^{(i-1)} \quad(i=2,\dots,k)
\end{align*}
\end{claim}
\begin{proof}
    We construct vectors $\boldsymbol{u}^{(2)},\ldots,\boldsymbol{u}^{(k)}$ in order using the following operations: 
    \begin{align*}
   \boldsymbol{\tilde{u}}^{(i)} &= \boldsymbol{v}^{(i)}- \;\sum_{j=1}^{i-1}\frac{\langle \boldsymbol{v}^{(i)},\boldsymbol{u}^{(j)}\rangle}{\|\boldsymbol{u}^{(j)}\|^2}\,\boldsymbol{u}^{(j)}\\
    \boldsymbol{u}^{(i)} &= \frac{\boldsymbol{\tilde{u}}^{(j)}}{1- \sum_{j=1}^{i-1}\frac{\langle \boldsymbol{v}^{(i)},\boldsymbol{u}^{(j)}\rangle}{\|\boldsymbol{u}^{(j)}\|^2}}\ .
    \end{align*}

  We establish the claim by induction on $i$.
  The claim clearly holds for $i=1$. 

  Assume it holds for $i$. Then
  $\boldsymbol{u}^{(i)}= \sum_{j=1}^i \gamma_{ij} \boldsymbol{v}^{(j)}$, where $\sum_{j=1}^i \gamma_{ij}$=1. 
  
  Therefore for all $h>i$,
  \begin{align*}
     \langle \boldsymbol{v}^{(h)},\boldsymbol{u}^{(i)}\rangle &=  \sum_{j=1}^i \gamma_{ij} \langle \boldsymbol{v}^{(h)} \boldsymbol{v}^{(j)} \rangle = \sum_{j=1}^i \gamma_{ij} = 1\ .
  \end{align*}

  Therefore
  \[
  \boldsymbol{\tilde{u}}^{(i+1)} = \boldsymbol{v}^{(i+1)}- \;\sum_{j=1}^{i}\frac{1}{\|\boldsymbol{u}^{(j)}\|^2}\,\boldsymbol{u}^{(j)}
  \]

  Using orthogonality of $(\boldsymbol{u}^{(j)})_{j\leq i}$ we obtain:
\[
  \|\boldsymbol{\tilde{u}}^{(i)}\|_2^2 = \|\boldsymbol{v}^{(i)}\|^2_2- \;\sum_{j=1}^{i-1}\frac{1}{\|\boldsymbol{u}^{(j)}\|_2^2}\geq \|\boldsymbol{v}^{(i)}\|^2_2- \;\sum_{j=1}^{i-1}\frac{1}{i\, (1+\ln k)} \geq \|\boldsymbol{v}^{(i)}\|^2_2 - 1\ .
  \]
  (using an upper bound on the Harmonic sum).
  
  The scaling of $\boldsymbol{\tilde{u}}^{(i)}$ ensures that $\boldsymbol{u}^{(i)}$ is an affine combination of $\boldsymbol{v}^{(i)}$ and $(\boldsymbol{u}^{(j)})_{j<i}$. Since by induction each $\boldsymbol{u}^{(j)}$ for $j<i$ is an affine combination of $(\boldsymbol{v}^{(h)})_{h\leq j}$, combining we obtain that so is $\boldsymbol{u}^{(i)}$.

  The scale factor is $s_i = 1- \sum_{j=1}^{i}\frac{1}{j\, (1+\ln k)} \in (0,1]$. Therefore,
  \[
  \|\boldsymbol{u}^{(i)}\|_2^2 = \frac{1}{s_i^2}
  \|\boldsymbol{\tilde{u}}^{(i)}\|_2^2 \geq \|\boldsymbol{\tilde{u}}^{(i)}\|_2^2 \geq 
  \|\boldsymbol{v}^{(i)}\|^2_2 - 1\ .
\]
  
\end{proof}

\begin{proof}[Proof of \cref{sigmaTlower:lemma}]
We first give a characterization of $\sigma^2_T(h,[m+1]\setminus\{h\})$ that we will use, and follows from a similar argument to the proof of \cref{signalsufficient:lemma}. Let $G \in \mathbb{R}^{k\times k}$ be invertible so that
the column $GA_{\bullet h}$ has only values in $\{0,1\}$ and the rows of the submatrix  $GA_{\bullet,[m+1]\setminus\{h\}}$ with column $h$ removed are orthogonal.
Then
\begin{equation} \label{sigmachar:eq}
    \frac{1}{\sigma^2_T(h,[m+1]\setminus\{h\})} = \sum_{i : GA_{ih}=1} \frac{m}{\|GA_{i,[m+1]\setminus\{h\}} \|_2^2}\ .
\end{equation}

We assume, without loss of generality, that the rows of 
$A$ are either orthogonal or $\boldsymbol{0}$.
From \cref{fragilecolumns:claim}, it follows that most columns of $A$ are fragile. 

We now fix a nonzero fragile column $h$. We scale
the rows so that $A_{\bullet h}$ are in $\{0,1\}$. 
Note that the fragility of columns is invariant to rescaling of rows. From fragility, and with the rescaling, the squared norms of the rows in increasing order are at least $\frac{im}{10k\log k}$.

We now consider the submatrix $B=A_{L_h,[m+1]\setminus\{h\}} \in \mathbb{R}^{k\times m}$ of $A$ with column $h$ removed and all rows in which column $h$ was not active are removed.
Let the rows of $B$ be $\boldsymbol{v}^{(i)}$ and observe that from orthogonality of the rows of $A$ and from
the fragility of $h$, the vectors satisfy the conditions in the statement of \cref{nearlytofully:claim}. Therefore by applying the claim we obtain a matrix $B'$ (with row vectors $\boldsymbol{u}^{i}$ that are orthogonal, are affine transformations of the rows of $B$, and that 
\[
\sum_i \frac{1}{\|B'_{i\bullet}\|^2_2}  \leq  \sum_i \frac{10k\log k}{im} = O(\frac{k\log k}{m})\ .
\]

It follows that if we applied the same transformations with the column $h$, the column would be invariant. The matrix $A'$ with column $A$ with $B'$ substituting the matrix $B$ could be obtained using the same transformation. This matrix satisfies the conditions of the characterization and by applying \eqref{sigmachar:eq} we obtain
 $\sigma^2_T(h,[m+1]\setminus\{h\} = \Omega(\frac{1}{k\log k})$ and this concludes the proof.

\end{proof}

\section{Proof of the \nameref{gain:lemma}} \label{gainproof:sec}
This section contains a restatement and proof of \cref{gain:lemma}.

Because $T_{h,M}$ is a sufficient statistic for $w$, the distribution of the sketch $A\boldsymbol{v}$, conditioned on $T_{h,M}(A\boldsymbol{v})= (G^{(h,M)}A\boldsymbol{v})_1=\tau$, does not depend on $w$. 
Let $f_{\tau}:\mathbb{R}^k$ be the density function of this distribution.

We can thus express the 
expected value of $\psi$, conditioned on the value of the statistic $T_{h,M}(\boldsymbol{y})=\tau$, (as it does not depend on the signal value $w$):
\begin{equation} \label{Psi:eq}
    \Psi(\tau) := \int_{\mathbb{R}^k} \psi(\boldsymbol{y}) f_{\tau}(\boldsymbol{y}) \, d\boldsymbol{y}\ . 
\end{equation}

We express the error rate of $\psi$ (see \cref{signalcorrect:def}) in terms of $\Psi$ and $\sigma_T^2$:  
\begin{align} 
\Err(\psi) &=  \int_a^1 \int_{\mathbb{R}} \frac{\Psi(w + x) + 1}{2} \,  \varphi_{0, c^2 \sigma_T^2}(x) \, dx \, \nu(w) \, dw \label{Psierror:eq} \\
&+ \int_{1+\alpha}^b \int_{\mathbb{R}} \frac{1-\Psi(w + x)}{2} \,  \varphi_{0, c^2 \sigma_T^2}(x) \, dx \, \nu(w) \,  dw\ ,\nonumber
\end{align}
where $\nu(w)$ be the density function of the distribution $W$ on the signal $w$ (see \cref{signaldensity:def}) and $\phi_{0,c\sigma_T^2}$ is the distribution of the deviation (see \cref{signalsufficient:lemma}).



\gainlemma*
\begin{proof}
We bound from below the expected value, over our query distribution, of 
\[
\psi(A\boldsymbol{v})(T_{h,M}(A\boldsymbol{v})-w)= c \psi(A\boldsymbol{v}) \Delta_{h,M}(A\boldsymbol{u}).\]

We express the expected value of $\psi$ over our query distribution, conditioned on the deviation 
$\Delta_{h,M}(cA\boldsymbol{u}) =  x$:
\begin{align}
\E\left[\psi(A\boldsymbol{v}) \mid \Delta_{h,M}(cA\boldsymbol{u}) = x \right] &=
   \frac{1}{\varphi_{0,c^2\sigma^2_T}(x)} \cdot \int_{a}^{b} \Psi(w+x) \varphi_{0,c^2\sigma^2_T}(x)\, \nu(w)\, dw \nonumber \\ &= \int_{a}^{b} \Psi(w+x) \, \nu(w)\, dw\nonumber \\
   &= \int_{a+x}^{b+x} \Psi(w)\, \nu(w-x)\, dw
   \ . \label{condPsi:eq}
\end{align}

We express the expected value of $\psi(A\boldsymbol{v}) \Delta_{h,M}(c A\boldsymbol{u})$ over our query distribution. 
 
\begin{align}
\lefteqn{\E_{\boldsymbol{v}}[\psi(A\boldsymbol{v}) \Delta_{h,M}(c A\boldsymbol{u})]} \nonumber \\
&= \E_{x \sim\mathcal{N}(0,c^2 \sigma_T^2)} x\, \E_{\boldsymbol{v}}\left[\psi(A\boldsymbol{v}) \mid \Delta_{h,M}(cA\boldsymbol{u}) = x \right] \nonumber \\
&= \int_{-\infty}^{\infty} x \varphi_{0,c^2\sigma^2_T}(x) 
\left( \int_{a+x}^{b+x} \Psi(w)\,\nu(w-x) \, dw \right) d x \qquad \text{; using  \eqref{condPsi:eq} } \nonumber\\
&= \int_{0}^{\infty} x \varphi_{0,c^2\sigma^2_T}(x) 
\left( \int_{a+x}^{b+x} \Psi(w)\,\nu(w-x) \, dw - \int_{a-x}^{b-x} \Psi(w)\,\nu(w+x) \, dw \right) d x \label{intofdiff:eq}
\end{align}
(using that $\varphi$ is a symmetric function).

We separately bound the parts of the integral in \eqref{intofdiff:eq} due to $x\in [\alpha,\infty]$ and $x\in [0,\alpha]$.

We first consider $x\in [\alpha,\infty]$.
Since $\nu$ is a density function and $|\Psi|\leq 1$, the absolute value of the difference expression is bounded by $2$.
Hence,
\begin{align}
\left| \text{\eqref{intofdiff:eq} due to $x\in [\alpha,\infty]$} \right| \nonumber 
\leq 2 \int_{\alpha}^\infty x \varphi_{0,c^2\sigma^2_T}(x)\, dx= \frac{2}{\sqrt{2\pi}}\exp(-(\alpha/(c\sigma_T))^2/2) \ .\label{highsigmacont:eq}
\end{align}

We now bound the contribution to \eqref{intofdiff:eq} due to $x\in [0,\alpha]$. We use the following:
\begin{claim}
    For $x\in [0,\alpha]$ and assuming $\Err(\psi)\leq 0.05$ and our setting of $a,b$,
    \begin{align}
        \int_{a+x}^{b+x} \Psi(w)\,\nu(w-x) \, dw - \int_{a-x}^{b-x} \Psi(w)\,\nu(w+x) \, dw = \Omega(c\,\alpha^{-1}\, x)\ .
    \end{align}
\end{claim}
\begin{proof}
We break the range of integration into parts according to the 
density function $\nu$ (see \cref{signaldensity:def}) and bound 
each part.

\begin{align}
\lefteqn{\int_{a+x}^{b+x} \Psi(w)\,\nu(w-x) \, dw - \int_{a-x}^{b-x} \Psi(w)\,\nu(w+x) \, dw }\\
&= -\int_{a-x}^{a+x} \Psi(w)\,\nu(w+x) \, dw   \label{leftend:eq}  \\
&+ \int_{a+x}^{1-x} \Psi(w)\,(\nu(w-x)-\nu(w+x))\, dw \label{leftslope:eq}\\
&+  \int_{1-x}^{1+x} \Psi(w)\,(\nu(w-x)-\nu(w+x))\, dw \label{leftofflat:eq} \\
&+  \int_{1+x}^{1+\alpha-x} \Psi(w)\,(\nu(w-x)-\nu(w+x))\, dw \label{flat:eq}\\
&+  \int_{1+\alpha-x}^{1+\alpha+x} \Psi(w)\,(\nu(w-x)-\nu(w+x))\, dw  \label{rightofflat:eq} \\
&+  \int_{1+\alpha+x}^{b-x} \Psi(w)\,(\nu(w-x)-\nu(w+x))\, dw  \label{downslope:eq}\\
&+  \int_{b-x}^{b+x} \Psi(w)\,\nu(w-x)\, dw  \label{rightend:eq} 
\end{align}

We now bound each part. We use that $C=\Theta(c/\alpha)$ and $(1-a)$,$b-(1+\alpha) = \Theta(\alpha/c)$.
\begin{align*}
 |\text{\eqref{leftend:eq}}| &\leq   C\frac{4x^2}{1-a} = O(c^2 \, \alpha^{-2} \, x^2) \\
    \text{\eqref{leftslope:eq}} &= \int_{a+x}^{1-x} \Psi(w)\,(-2x \frac{C}{1-a})\, dw = -2x \frac{C}{1-a} \int_{a+x}^{1-x} \Psi(w)\,dw \geq 0.1\, c \, \alpha^{-1}\, x \\
|\text{\eqref{leftofflat:eq}}|, |\text{\eqref{rightofflat:eq}}| &\leq 2Cx^2 = O(c\alpha^{-1} x^2)\\
    \text{\eqref{flat:eq}} &= 0 \\
        \text{\eqref{downslope:eq}} &= 2x \frac{C}{b-(1+\alpha)}\int_{1+\alpha+x}^{b-x} \Psi(w)\, dw \geq 0.1\, x\, c\, \alpha^{-1} \\
    |\text{\eqref{rightend:eq}}| &\leq   C\frac{4x^2}{b-(1+\alpha)}= O(c^2 \, \alpha^{-2} \, x^2)\\
\end{align*}

Our bounds for \eqref{leftslope:eq} and \eqref{downslope:eq}  use that for $\Psi$ with error at most $\delta<0.05$ and $x\in [0,\alpha]$, 
it holds (with our choices for $a,b$) that
$\int_{a+x}^{1-x} \Psi(w)dw < -0.1 (1-a)$,
$\int_{1+\alpha+x}^{b-x} \Psi(w)dw > 0.1 (b-(1+\alpha))$.
\end{proof} 

\begin{align}
 \text{\eqref{intofdiff:eq} due to $x\in [0,\alpha]$} = \int_0^{\alpha}x \varphi_{0,c^2\sigma^2_T}(x)  \, \Omega(Cx) dx =   \Omega(c^3\alpha^{-1}\sigma_T^2)\ .
  \label{lowsigmacont:eq}
\end{align}

Combining \eqref{lowsigmacont:eq} and \eqref{highsigmacont:eq} we establish the claim in the statement of the Lemma: 
    $\text{\eqref{intofdiff:eq}} = \Omega(\alpha^{-1}\, c^3\sigma_T^2)$.

\end{proof}

\section{Signed sum of Gaussian vectors }

\begin{restatable}[Upper bound for the signed sum of Gaussian vectors]{lemma}{signed-gauss-upper}
\label{lem:signed-gauss-upper}
Let $X_{1},\dots ,X_{r}\stackrel{\mathrm{i.i.d.}}{\sim}\mathcal N_m(0,1)$ in
$\mathbb{R}^{m}$.
Define
\[
M\;:=\;\max_{s\in\{\pm1\}^{r}}\;
      \Bigl\lVert\;\sum_{i=1}^{r}s_{i}X_{i}\Bigr\rVert_{2}.
\]
Then for every $t>0$
\[
\Pr\!\Bigl[\,M\;\le\;(\sqrt r+\sqrt m+t)\,\sqrt r\,\Bigr]
\;\;\ge\;1-2e^{-t^{2}/2}.
\]
In particular, 
\[
M\;\le\;2r+\sqrt{rm}
\qquad\text{with probability }1-e^{-r/2}.
\]
\end{restatable}
\begin{proof}
Collect the random vectors into the Gaussian matrix
\(
X\;=\;\bigl[X_{1}\,\dots\,X_{r}\bigr]\in\mathbb{R}^{m\times r}.
\)
For any sign vector $\boldsymbol{s}\in\{\pm1\}^{r}$ we have
\[
\sum_{i=1}^{r}s_{i}X_{i} \;=\;X\boldsymbol{s},
\qquad\text{hence}\qquad
\bigl\lVert X\boldsymbol{s}\bigr\rVert_{2}
      \;\le\;
      \sigma_{\max}(X)\,\lVert \boldsymbol{s}\rVert_{2}
      \;=\;\sigma_{\max}(X)\,\sqrt r,
\]
where $\sigma_{\max}(X)$ is the largest singular value of $X$.
Since this holds for all $\boldsymbol{s}$, we get 
\begin{equation}
\label{eq:M-by-sigma-max}
M \leq \sigma_{\max}(X)\,\sqrt{r}\, .
\end{equation}

For an $m\times r$ matrix with i.i.d.\ $\mathcal N(0,1)$ entries, applying
the standard concentration bound (see, e.g., Theorem~4.4.5 in \citet{Vershynin18}) we obtain
\[
\Pr\!\bigl[\sigma_{\max}(X)\;\ge\;\sqrt r+\sqrt m+t\bigr]
        \;\le\;2e^{-t^{2}/2}\qquad(t>0).
\]
Combining this bound with \eqref{eq:M-by-sigma-max} yields the first stated probability bound.
For the simplified bound, we plug-in $t=\sqrt r$ in the above.
\end{proof}

\end{document}